%% file: arxiv.tex
\documentclass[letterpaper]{article}
\pdfoutput=1
\usepackage[square]{natbib}

\usepackage{fullpage}
\usepackage[utf8]{inputenc}

\usepackage{graphicx}
\usepackage[utf8]{inputenc}
\usepackage{amssymb}
\usepackage{amsthm}
\usepackage[letterpaper, margin=1in]{geometry}
\usepackage{mathtools}
\usepackage{xcolor}

\DeclareMathOperator*{\argmin}{arg\,min}
\usepackage{algorithm}
\usepackage{algorithmic}
\usepackage{footnote} 
\usepackage{tcolorbox}

\newtheorem{theorem}{Theorem}
\newtheorem{corollary}{Corollary}
\newtheorem{lemma}{Lemma}
\newtheorem{assumption}{Assumption}
\newtheorem{definition}{Definition}

\def\ba{{\boldsymbol{a}}}

\def\bx{{\boldsymbol{x}}}

\def\bv{{\boldsymbol{v}}}
\def\bu{{\boldsymbol{u}}}

\def\bA{{\boldsymbol{A}}}
\def\bX{{\boldsymbol{X}}}
\def\bY{{\boldsymbol{Y}}}
\def\bU{{\boldsymbol{U}}}
\def\bV{{\boldsymbol{V}}}
\def\bW{{\boldsymbol{W}}}

\def\Ocal{{\mathcal{O}}}
\def\gtil{{\tilde{\nabla}}}
\DeclareMathOperator*{\expect}{\mathbb{E}}

\newcommand{\shortrnote}[1]{ &  &  & \text{\footnotesize\llap{#1}}}


\usepackage{titlesec}
\usepackage{setspace}

\title{Communication-Efficient Asynchronous Stochastic Frank-Wolfe \\ over Nuclear-norm Balls}
\author{Jiacheng Zhuo, Qi Lei, Alexandros G. Dimakis, Constantine Caramanis
\\
University of Texas at Austin \\
{\tt\small \{jzhuo@, leiqi@ices., dimakis@austin., constantine@\}utexas.edu}
}

\begin{document}

%

%

\maketitle

\begin{abstract}
Large-scale machine learning training suffers from two prior challenges, specifically for nuclear-norm constrained problems with distributed systems: the synchronization slowdown due to the straggling workers, and high communication costs.
In this work, we propose an asynchronous Stochastic Frank Wolfe (SFW-asyn) method, which, for the first time, solves the two problems simultaneously, while successfully maintaining the same convergence rate as the vanilla SFW. 
We implement our algorithm in python (with MPI) to run on Amazon EC2, 
and demonstrate that SFW-asyn yields speed-ups almost linear to the number of machines compared to the vanilla SFW. 

\end{abstract}

\input{1intro.tex}

\input{3prelim.tex}

\input{4algo.tex}

\input{5analysis.tex}

\input{6experiment.tex}

\input{7conclusion.tex}

\bibliographystyle{apalike}

\begingroup
\setstretch{0.97}
\bibliography{arxiv}
\endgroup

\clearpage
\appendix

\input{supp_proof.tex}
\input{supp_SVRF_asyn.tex}
\input{supp_simulation_setup.tex}
\input{supp_more_simulation.tex}
\end{document}

%% file: 1intro.tex
\section{Introduction}

We consider the problem of minimizing a convex and smooth matrix function 
subject to a nuclear norm constraint:
\begin{equation}
    \min_{\|\bX\|_* \leq \theta} F(\bX) = \frac{1}{N}\sum_i^N f_i (\bX),
\label{eqn:main-problem}
\end{equation}
where $F$  maps from $\mathbb{R}^{D_1 \times D_2}$ to $\mathbb{R}$, and $\|\bX\|_*$ denotes the 
nuclear norm of $\bX$, i.e., the sum of its singular values. We view the objective $F$ as a summation of $N$ sub-problems $f_1,f_2,\cdots f_N$. This formulation covers various important machine learning applications, including matrix completion, matrix sensing, multi-class and multi-label  classification, affine rank minimization, phase retrieval problems, and many more (see e.g. \citep{candes2015phase, allen2017bfw} and references therein). We are especially interested in the 
large sample setting, i.e., when the total number of subproblems (data examples) $N$ and model size $D_1, D_2$ are large.

Projection-free algorithms are frequently used for solving problem \ref{eqn:main-problem}. As first proposed by Frank and Wolfe \citep{frank1956algorithm}, 
one can compute and update the parameters along the gradient descent direction 
while remaining within the constraint set. This is thus called the Frank-Wolfe (FW) algorithm or conditional gradient descent.  The nuclear norm constraint requires the computation of the leading left and right singular vectors of the gradient matrix, with complexity of $\Ocal(D_1D_2)$. A natural alternative among first-order algorithms is Projected Gradient Descent (PGD). The projection step, however, requires a full SVD per iteration, with much higher complexity: $\mathcal{O}(D_1 D_2 \cdot \min(D_1, D_2))$. This makes PGD computationally expensive for  large-scale datasets. 

Classical FW runs on a single machine and passes the whole dataset in each iteration. 
As datasets increase and Moore's Law is slowing down \citep{simonite2016moore}, the move towards stochastic variants of FW has become imperative \citep{hazan2016variance}. In the meantime, it is also crucial to study distributed FW implementations, where gradient computation and aggregation is parallelized across multiple worker nodes \citep{zheng2018distributed, bellet2015distributed}. Although distributed computation boosts the amount of data that can be processed per iteration, it may introduce prohibitive communication overhead, particularly in the big data setting with high-dimensional parameters. Specifically, there are two challenges: (1) the stragglers dominate each iteration due to synchronization \citep{liu2015asynchronous, passcode15, reddi2015variance, recht2011hogwild, tandon2017gradient}; (2) there is a large communication cost per iteration ($\Ocal(D_1 D_2)$) due to gradient and parameter sharing. 
For problem (2), (stochastic) FW is natural and appealing: the low-rank updates can be represented as a few vectors and they require fewer bits to transmit than the partial gradients. However, naively aggregating the low-rank updates from the workers does not yield an algorithm that converges, as the Singular Vector Averaging algorithm in the work of \citep{zheng2018distributed}. While natural and promising, there are several significant technical obstacles towards designing an efficient and provably convergent distributed FW algorithm. To the best of our knowledge, no prior work has successfully addressed this. 

We thus raise and answer the following two questions in this work:

\textit{ (1) Can FW type algorithms \textbf{provably} run in an asynchronous (lock-free) manner, so that the iteration time is not dominated by the slowest worker?} 

\textit{ (2) Can we improve the communication cost of running FW in the distributed setting?} 

In this work, we are able to answer both of these questions affirmatively, by proposing and analyzing an asynchronous version of Stochastic Frank-Wolfe algorithm, which we call \textbf{SFW-asyn}.

\subsection*{\bf Our contributions.}

We propose a Stochastic Frank-Wolfe algorithm (SFW-asyn) that is communication-efficient and runs asynchronously. Apart from the lock-free benefit, SFW-asyn reduces the communication cost from $\Ocal(D_1 D_2)$ as in most prior work, to $\Ocal(D_1 + D_2)$. 

Moreover, we establish that SFW-asyn (Algorithm \ref{alg:SFW-asyn}) enjoys a $\mathcal{O}(1/k)$ convergence rate, which is in line with the standard SFW.

We further analyze how vanilla SFW and SFW-asyn perform when the mini-batch size is set to a constant, while the previous analysis on SFW requires an increasing mini-batch size that is quadratic in the iteration number \citep{hazan2016variance}.

Our theoretical analysis framework is general for other Frank-Wolfe type method. We specifically provide the extension for Stochastic Variance Reduced Frank-Wolfe (SVRF) \citep{hazan2016variance}.

Finally, we use Amazon EC2 instances to provide extensive performance evaluation of our algorithm, and comparisons to vanilla SFW and distributed SFW on two tasks, matrix sensing and learning polynomial neural networks. Our results show that with our communication-efficient and asynchronous nature, we achieve speed-ups almost linear in the number of distributed machines. 

\subsection*{\bf Related work.}

We mainly overview two different lines of work that are respectively related to the distributed Frank-Wolfe algorithms, and the asynchronous manner in distributed learning. 

{\bf Distributed Frank-Wolfe Algorithms.}
For general network topology, 
\cite{bellet2015distributed} proposes a distributed Frank-Wolfe algorithm for $\ell_1$ and simplex constraints, but not for the nuclearn-norm constraint. 

On a master-slave computational model, \cite{zheng2018distributed} proposes to distribute the exact gradient computation among workers and aggregate the gradient by performing $\Ocal(t)$ distributed power iterations at iteration $t$. This involves even more frequent synchronizations than the vanilla distributed Frank-Wolfe method 
and will be hindered heavily by staleness. Specifically, to run $T$ iterations, the distributed FW algorithm proposed by \citep{zheng2018distributed} requires a total communication cost of $\Ocal(T^2(D_1 + D_2)$, while SFW-asyn only requires $\Ocal(T(D_1 + D_2))$. Furthermore, since we consider the case when data samples are very large, stochastic optimization is much more efficient than exact gradient computation, and the whole sample set might not fit into memory. Therefore this work and other full-batch FW methods are beyond the interests of our paper.

Quantization techniques have also been used to reduce the communication cost for SFW, like \citep{zhang2019quantized}, who proposed a novel gradient encoding scheme (s-partition) to reduce iteration communication cost. However their results are for $\ell_1$ constrained problem, and are thus not directly comparable to our results. 

{\bf Asynchronous Optimization.}  
There is a rich previous effort in the research on asynchronous optimization, such as asynchronous Stochastic Gradient Descend (SGD) \citep{recht2011hogwild}, asynchronous Coordinate Descend (CD) \citep{liu2015asynchronous, passcode15}, and asynchronous Stochastic Variance Reduced Gradient (SVRG) \citep{reddi2015variance}. \cite{mania2015perturbed} 
proposes a general analysis framework, which however, is not applicable for SFW. 
\cite{wang2014asynchronous} 
proposes Asynchronous Block Coordinate FW and achieves sub-linear rate. 
They assume the parameters are separable by coordinates, and hence different from our problem setting.

{\bf A concurrent work.}
We acknowledge a concurrent but different work \cite{guasynchronous} that also studies asynchronous SFW.
Firstly we have different problem settings - we study how to minimize matrix to scalar functions (problem \eqref{eqn:main-problem}) under master-slave distributed computational setting, while they consider vector to scalar functions under the shared memory model.
For the algorithm, We propose an novel algorithm that is simultaneously straggler-resilient and communication efficient, while they only tackle the straggler problem.
For the analysis, our work further studies the convergence under any constant batch-size which is more practical.
We also show that how staleness may `help' by reducing the batch-size per iteration, enabling us to match or even improve the overall complexity of vanlilla SFW, which is not presented in this concurrent work.










%% file: 3prelim.tex
\section{Preliminary}
\subsection{Frank-Wolfe method}
We start by reviewing the classical (Stochastic) Frank-Wolfe (FW/SFW) algorithms. Consider a general constrained optimization problem $\min_{\bX \in \Omega} \left\{F(\bX) := \frac{1}{N}\sum_i^N f_i (\bX)\right\}$. FW proceeds in each iteration by computing:
\begin{align}
    \bU_k &= \argmin_{\bU \in \Omega} \quad \langle  \nabla F(\bX_{k-1}), \bU \rangle,
    \label{eqn:FW_linearopt}\\
    \bX_k &= (1 - \eta_k) \bX_{k-1} + \eta_k \bU_k.
    \label{eqn:FW_update}
\end{align}
and SFW replaces the full batch gradient by a mini-batch gradient
\begin{align}
    \bU_k &= \argmin_{\bU \in \Omega} \quad \left \langle   \frac{1}{|\mathcal{S}_k|} \sum_{i\in \mathcal{S}_k}\nabla f_i (\bX), \bU \right \rangle, 
    \label{eqn:SFW_linearopt}\\
    \bX_k &= (1 - \eta_k) \bX_{k-1} + \eta_k \bU_k.
    \label{eqn:SFW_update}
\end{align}
where $\mathcal{S}_k$ is a randomly selected index set. We will refer to Eqns \eqref{eqn:FW_linearopt} and \eqref{eqn:SFW_linearopt} as linear optimization steps.

The complexity of the Frank-Wolfe type methods depends on that of the linear optimization steps. Fortunately, for a wide class of constraints $\Omega$, the linear optimization, with the projection-free nature, is of relatively low computational complexity. When $\Omega$ is $\| \bX \|_* \leq 1$ for instance, the linear optimization step is to compute the leading left singular vector $\bu$ and the right singular vector $\bv$ of the negative gradient (or mini-batch gradient), and return $\bu \bv^T$.

\subsection{Computational Model}
\begin{figure}[ht]
\begin{center}
\includegraphics[width=0.35\linewidth]{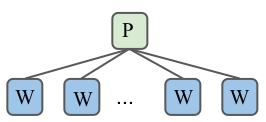}
\includegraphics[width=0.38\linewidth]{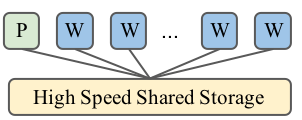}
\end{center}
\caption{{\em Illustration on different distributed frameworks.} `P' stands for Parameter-server, and `W' denotes workers. We focus on the master-slave paradigm (left), where all the workers are directly connected to the master node. 
Our algorithms and analysis can be easily modified for a shared-memory model (right), where all the machines have high speed access to one storage device, instead of being connected by network.}
\label{fig:comp_arch} 
\end{figure}

In this work we focus on the master-slave computational model, which is commonly used for distributed applications \citep{li2014scaling}. In the master-slave model, we have $W$ worker (slave) nodes and one Master node (parameter server). Each worker has access to all the data and therefore is able to compute $f_i(\bX)$ and $\nabla f_i(\bX)$ when $\bX$ is given. The Master node has direct connection to each worker node and serves as the central coordinator by maintaining the model parameters. 

Although synchronous framework is commonly used for iterative algorithms as they require minimal modification of the original algorithms, the running time of each iteration is confined by the slowest worker \citep{li2014communication, dean2012large}. This is why {\em asynchronous distributed algorithms} are desirable: 
each machine just works on their assigned computations, and communicates by sending signals or data asynchronously \citep{recht2011hogwild,liu2015asynchronous, passcode15}. The benefits of asynchronous algorithms are obvious: the computational power for each machine is fully utilized, since no waiting for the straggling workers is necessary. Unfortunately, the asynchronous version of an iterative algorithm may suffer from the effect of staleness.  
\begin{definition}
In the asynchronous computational model, workers may return stale gradients that were evaluated at an older version of the model $\bX$; we call that there is a staleness of $\tau$ if $\bX_{t+1}$ is updated according to $\bX_{t-\tau}$.
\end{definition}
\begin{figure}[H]
\begin{center}
\includegraphics[width=0.60\linewidth]{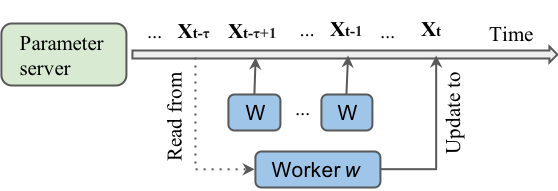}
\end{center}
\caption{Delay (staleness) for a worker $w$ is the number of updates by other workers between two updates of $w$. }
\label{fig:staleness} 
\end{figure}
As shown in Figure \ref{fig:staleness}, staleness happens when a worker tries to update the parameter server with out-dated information. In other words, when each worker deals with its own computation, other workers may have updated the model of the parameter server already. Staleness on a worker $w$ during $t$ execution cycle is a series of $t$ random variables, and we call this \textit{staleness process} for the worker $w$.  We require mutual independence on the staleness process, and the independence between the staleness process among workers and the sampling process of a worker.
\begin{assumption}
The staleness process of worker $w_i$ and the staleness process of worker $w_j$ are independent, for all $i \neq j$.
\label{assumption:independent1}
\end{assumption}
\begin{assumption}
The staleness process of worker $w$ is independent with the sampling process (for the stochastic gradient computation) on the worker $w$.
\label{assumption:independent2}
\end{assumption}
These two assumptions are standard in the analysis of asynchronous algorithms \citep{mitliagkas2016asynchrony}.

Although we focus on the master-slave computational model, our algorithm and analysis can be easily modified for the shared-memory computational model \citep{recht2011hogwild}, where all the workers have high speed access to one shared storage device. A typical example is a server with multiple CPUs, and these CPUs have access to a same piece of memory. 
In the shared-memory model, communication is no longer an issue. Our proposal mostly benefits from the lock-free nature under this setting. 

\subsection{Notation}
We use bold upper case letter for matrices, e.g. $\bX$, and bold lower case letter for vectors, e.g. $\bx$. We define the matrix inner product as $\langle \bX, \bY \rangle = \textrm{trace} (\bX^T \bY)$.
A function $F$ is convex on set $\Omega$ if
    $$F(\bY) \geq F(\bX) + \langle \nabla F(\bX), \bY-\bX \rangle, \forall \bX, \bY \in \Omega.$$
Similarly, $F$ is $L$-smooth on set $\Omega$ when 
$$   \| \nabla F(\bX) - \nabla F(\bY)\| \leq L \| \bX - \bY\|,\forall \bX, \bY \in \Omega.$$
We use $F^*$ to denote $\min_{\|\bX\|_* \leq \theta} F(\bX)$, namely the minimal value that $F$ can achieve under the constraints. Note there could be multiple optimal $\bX$, and hence we use $\bX^*$ to denote one of them, unless otherwise specified.

Specifically for Problem \ref{eqn:main-problem}, define $D$ as the diameter of the constraint set: $D = \max_{\|\bX\|_* \leq \theta, \|\bY\|_* \leq \theta} \| X - Y\|_F$, and define $G$ as the variance of the stochastic gradient: $G^2 \leq \expect \| \nabla f_i (\bX) - \nabla F(\bX) \|_F^2$, for all $\|\bX\|_* \leq \theta$.

%% file: 4algo.tex
\section{Methodology}
\label{sec:algo}
\begin{algorithm}[t]
\caption{Distributed Stochastic Frank-Wolfe (SFW-dist) (A baseline method)}
\begin{algorithmic}[1]
\STATE \textbf{Input:} Max iteration count $T$; Step size $\eta_t$ and batch size $m_t$ for iteration $t$.
\STATE \textbf{Initialization:} Random $X_0$ s.t. $\|\bX_0\|_*=1$.
\FOR{iteration $k=1,2,\cdots, T$}
    \STATE
    Broadcast $\bX_{k-1}$ to all workers.
    \FOR{ each worker $w=1,2,\cdots, W$}
        \STATE
        Randomly sample an index set $S_w$ where $|S_w| = m_k / W$
        \STATE 
        Compute and send $ \sum_{i \in S_w} \nabla f_i (\bX_{k-1}) $  to the master
    \ENDFOR
    \STATE
    $\nabla_k = \sum_{w=1}^W \sum_{i \in S_w} \nabla f_i (\bX_{k-1})$
    \STATE
    $\bU_k \leftarrow \textrm{argmin}_{\|\bU\|_*\leq \theta} \langle \nabla_k , \bU \rangle$
    \STATE
    $\bX_{k} \leftarrow \eta_{k} \bU_k + (1-\eta_k) \bX_{k-1}$
\ENDFOR
\end{algorithmic}
\label{alg:SFW-dist}
\end{algorithm}

\begin{algorithm}[t]
\caption{Naive Asynchronous Stochastic Frank-Wolfe Method (Only for analysis)}
\begin{algorithmic}[1]
\STATE //  The Master Node 
\STATE \textbf{Input:} Max delay tolerance $\tau$; Max iteration count $T$; Step size $\eta_t$ and batch size $m_t$ for iteration $t$.
\STATE \textbf{Initialization:} Random $X_0$ s.t. $\|\bX_0\|_*=1$; broadcast $X_0$ to all the workers; the iteration count at the master node $t_m = 0$.
\WHILE{$t_m < T$}
    \STATE
    Wait until $\{\bU_w, t_w\}$ is received from a worker $w$.
    \STATE
    // $\bU_w$ is computed by worker $w$ according to $\bX_{t_w}$
    \STATE
    If $t_m - t_w > \tau$, abandon $\bU_w$ and continue.
    \STATE
    $t_m = t_m + 1$
    \STATE
    $\bX_{t_m} \leftarrow \eta_{t_m} \bU_{w} + (1-\eta_{t_m}) \bX_{t_m - 1}$
    \STATE
    // Update $\bX$ in the master node with $\bU_{t_w}$.
    \STATE
    Broadcast $\bX_{t_m}$ to all the workers
\ENDWHILE
\STATE // For each worker $w=1,2,\cdots, W$ 
\WHILE{No Stop Signal}
    \STATE 
    Receive from the Master $\bX_{t_m}$.
    \STATE
    // $t_m$ is the iteration count at the master node.
    \STATE
    $t_w = t_m$, $\bX_w = \bX_m$.
    \STATE
    // Update the local copy of $\bX$ and iteration count.
    \STATE
    Randomly sample an index set $S$ where $|S| = m_{t_w}$
    \STATE 
    $\bU_w \leftarrow \textrm{argmin}_{\|\bU\|_*\leq \theta} \langle \sum_{i \in S} \nabla f_i (\bX_w), \bU \rangle$ 
    \STATE
    Send $\{\bU_w, t_w\}$ to the Master node.
\ENDWHILE
\end{algorithmic}
\label{alg:SFW-asyn-naive}
\end{algorithm}

\begin{algorithm}[t]
\caption{Asynchronous Stochastic Frank-Wolfe Method (SFW-asyn)}
\begin{algorithmic}[1]
\STATE //  The Master Node
\STATE \textbf{Input:} Max delay tolerance $\tau$; Max iteration count $T$; Step size $\eta_t$ and batch size $m_t$ for each iteration $t$;
\STATE \textbf{Initialization:} Randomly initialize $\bX_0$ = $\bu_0 \bv_0^T$ s.t. $\|\bX_0\|_*=1$ and broadcast $\{\bu_0, \bv_0\}$ to all the workers; The iteration count at the master node $t_m = 0$.
\WHILE{$t_m < T$}
    \STATE
    Wait until $\{\bu_w, \bv_w, t_w\}$ is received from a worker $w$.
    // $\bu_w, \bv_w$ are computed according to $\bX_{t_w}$.
    \IF{$t_m - t_w > \tau$}
        \STATE
        Send $(\bu_{t_m}, \bv_{t_m}), ... (\bu_{t_w+1}, \bv_{t_w+1})$ to node $w$.
        \STATE
        \textbf{continue}.
    \ENDIF
    \STATE
    $t_m = t_m + 1$ and store $\{ \bu_w, \bv_w\}$ as $\bu_{t_m}$ and $\bv_{t_m}$ 
    \STATE
    Send $(\bu_{t_m}, \bv_{t_m}), ... (\bu_{t_w+1}, \bv_{t_w+1})$ to node $w$.
    \STATE
    $\bX_k \leftarrow \eta_k \bu \bv^T + (1-\eta_k) \bX_{k-1}$ 
    // Not run in real time;
    maintain a copy of $\bX_k$ for output only
\ENDWHILE
\STATE // For each worker $w=1,2,\cdots, W$ 
\WHILE{No Stop Signal}
    \STATE 
    Obtain $(\bu_{t_m}, \bv_{t_m}), ... (\bu_{t_w+1}, \bv_{t_w+1})$ from the master node.
    \STATE 
    Update the local copy of $\bX_{t_w}$ to $\bX_{t_m}$ 
    \STATE 
    // According to Eqn.\eqref{eqn:update_to_k};
    \STATE
    Update the local iteration count $t_w = t_m$.
    \STATE
    Randomly sample an index set $S$ where $|S| = m_{t_w}$
    \STATE 
    $\bu_w \bv_w^T \leftarrow \textrm{argmin}_{\|\bU\|_*\leq \theta} \langle \sum_{i \in S} \nabla f_i (\bX_{t_w}), \bU \rangle$ 
    \STATE
    send $\{\bu_w, \bv_w, t_w\}$ to the Master node.
\ENDWHILE
\end{algorithmic}
\label{alg:SFW-asyn}
\end{algorithm}

\begin{figure}[ht]
\begin{center}
\includegraphics[width=0.60\linewidth]{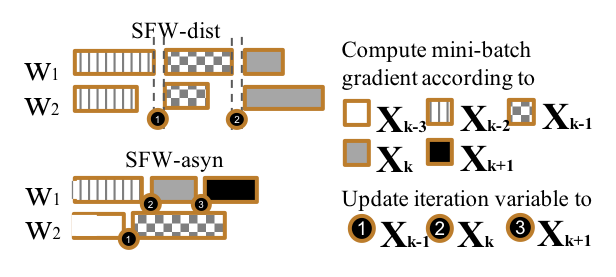}
\end{center}
\caption{{\em Illustration of simplified asynchronous computation scheme.} 'W' denotes workers, and 'X' denotes the iteration variable. All the subscript of 'X' refer to the iteration number at the master node. In SFW-asyn, the update '1' updates the model to $X_{k-1}$ with gradient computed according to $X_{k-3}$, and hence the delay is $1$. The delay is $1$ and $0$ for the update '2' and '3' respectively.}
\label{fig:asyn-sfw-illu}
\end{figure}

A natural way to deploy Stochastic Frank-Wolfe (SFW) on a master-slave distributed paradigm is explained as follows in four steps. For each iteration, we first have each worker $w = 1 ... W$ to compute $1/W$ portion of the mini-batch gradients. The master node next collects the mini-batch gradients from all the workers. It then computes the update for $\bX$, as in Eqn.\eqref{eqn:SFW_linearopt}, Eqn.\eqref{eqn:SFW_update}. Finally the master node broadcasts the updated $\bX$ to all the workers. 
We call this approach \textbf{SFW-dist} and describe it in Algorithm \ref{alg:SFW-dist}.
SFW-dist and its full batch variant serve as competitive baselines for previous state-of-the-arts \citep{zheng2018distributed}. 

However, SFW-dist has two obvious drawbacks: 
\textbf{high communication cost} and  \textbf{synchronization slow down.} 
Specifically, the master has to collect a gradient estimation from each worker for each iteration, while the size of the gradient is $\mathcal{O}(D_1 \cdot D2)$. 
When $D_1$ and $D_2$ are large, this communication cost is unaffordable \citep{zheng2018distributed}. 
On the other hand, the master has to wait to receive the gradient matrix from all workers in order to compute each update, and therefore the  time cost per iteration is dominated by the slowest worker.


We thus propose Asynchronous Stochastic Frank-Wolfe method (SFW-asyn), which simultaneously resolves the above two challenges.

For the ease of understanding, we gradually build up to our methodology and first introduce a naive version of asynchronous SFW in Algorithm \ref{alg:SFW-asyn-naive}. It only bares the benefit of asynchronous nature but no efficient communication yet. Algorithm \ref{alg:SFW-asyn-naive} is for illustration and analysis only, and not for implementation. Then we will explain how to reduce the iteration complexity to $\Ocal(D_1 + D_2)$ by making use of the low rank update nature of SFW. The whole algorithm, SFW-asyn, is described in Algorithm \ref{alg:SFW-asyn}.

As described in Algorithm \ref{alg:SFW-asyn-naive}, each worker computes a mini-batch gradient according to the latest $\bX$ it has access to, and sends the updates to the master. We denote by $t_w$ as the iteration count of $\bX$ that the worker $w$ is using to compute the updates. When the update from the worker $w$ reaches the master node, the $\bX$ in the master node reaches $t_m$ already, due to updates from the other workers. If the delay, defined as $t_m - t_w$, exceeds the maximum delay tolerance $\tau$, the master node will abondon this update. Otherwise, the master node updates the $\bX$ accordingly and broadcasts the new $\bX$ to all the workers. Algorithm \ref{alg:SFW-asyn-naive} does not have to wait for one particular slow worker in order to proceed, and therefore is resilient to the staleness of workers.

While Algorithm \ref{alg:SFW-asyn-naive} addresses the synchronization problem, the communication cost is still $\Ocal(D_1 D_2)$. The intuition of reducing the communication cost is that, all the updates matrix $\bU$ are rank one matrices, and can be perfectly represented by the outer products of two vectors. A natural idea is thus to transfer and store the vectors for potential updates, instead of the whole gradient matrices.  We ask the workers to transfer, instead of $\bU$, two vectors $\bu, \bv$ with $\bu \bv^T = \bU$, to the master. The master transfers back $\{\bu_{t_w}, \bv_{t_w}\}, ..., \{\bu_{t_m+1}, \bv_{t_m+1}\}$ , instead of $\bX_{t_m}$, to the worker $w$, and let the worker $w$ to update its own copy of $\bX_{t_w}$ to $\bX_{t_m}$ by recursively computing
\begin{equation}
\label{eqn:update_to_k}
    \bX_k = (1-\eta_k)\bX_{k-1} + \eta_k \bu_k \bv_k^T.
\end{equation}
From this practice, we reduce the communication cost to $\Ocal(D_1 + D_2)$.

Although we introduce a few extra operations that require $(D_1 D_2)$ computation as in the lines 3 of Algorithm \ref{alg:SFW-asyn}, we lifted the need to aggregate gradients from workers as in SFW-dist, which also requires $\Ocal(D_1 D_2)$ computation. Since the linear optimization (in this case, 1-SVD) requires $\Ocal(D_1 D_2)$ computation anyways, these operations do no change the asymptotic computation time.

Following the similar spirit, we propose the asynchronous Stochastic Variance Reduced Frank-Wolfe method (SVRF-asyn). We defer the full version (both asynchronous and communication efficient) to the Appendix in Algorithm \ref{alg:SVRF-asyn-naive} (not communication efficient, and only for analysis instead of implementation) and Algorithm \ref{alg:SVRF-asyn}.

\textbf{Communication Cost of SFW-asyn} For each iteration, one worker is communicating with the master. The worker sends a $\{ \bu_{t_w}, \bv_{t_w}, t_w\}$ to the master, and the master sends back an update sequence $\{\bu_{t_m}, \bv_{t_m}\}, ..., \{\bu_{t_w+1}, \bv_{t_w+1}\},$ to the worker. 
Consider the following amortized analysis. From iteration $1$ to iteration $T$, a worker $w$ received at most $T$ pairs of $\{\bu, \bv\}$ from the parameter server so that to keep its local copy of $\bX$ up-to-date. For all the $W$ workers, the total amount of message they send to the parameter server is $T$ pairs of $\{\bu, \bv\}$. Therefore, from iteration $1$ to iteration $T$, for all the $W$ workers, messages sent from the parameter server to the workers are at most with the size of $(TW(D_1 + D_2))$, and the messages sent from  workers to the parameter server are at most with the size of $(T(D_1 + D_2))$. For each iteration, on average, the communication cost is $\Ocal(W(D_1 + D_2))$. Since in the master-slave computational model, each worker has a direct connection channel with the parameter server, the communication cost along each channel is $\Ocal(D_1 + D_2)$ per iteration.


%% file: 5analysis.tex
\section{Theoretical Analysis}
\label{sec:analysis}

The analysis shown in this Section are developed under Assumption \ref{assumption:independent1} and \ref{assumption:independent2}. We first present the convergence results for SFW-asyn, and the discuss the convergence behavior in a more practical scenario where the batch size is fixed (or capped).

\begin{theorem}
\textbf{(The $\mathcal{O}(1/k)$ convergence rate of SFW-asyn)}. Consider problem \ref{eqn:main-problem}, where $F(\bX)$ is convex and $L$-smooth over the nuclear-norm ball. Denote $h_k = F(\bX_k) - F^*$, where $\bX_k$ is the output of the $k^{th}$ iteration of SFW-asyn. If for all $i < k$, $m_i = \frac{G^2 (i+1)^2}{\tau^2 L^2 D^2}$, $\eta_i = \frac{2}{i+1}$, $\tau < T/2$,
\begin{equation}
\expect \left[ h_k\right] \leq \frac{(3\tau + 1)\cdot 4 L D^2}{k+2}    
\end{equation}
\label{thm:conv-SFW-asyn}
\end{theorem}

We defer the proof to the appendix. The $\mathcal{O}(1/k)$ convergence rate of SFW-asyn matches the convergence rate of the original SFW \citep{hazan2016variance}.
While there is a $\Ocal(\tau)$ slowdown in the convergence rate comparatively, we only require a $\Ocal(i^2/\tau^2)$ batch size for each iteration $i$, instead of  $\Ocal(i^2)$ as in the original SFW.

The requirement of increasing batch size for SFW is that, there is no self-tuning gradient variance, and decreasing step size is not enough to control the error introduced by stochastic gradient. 

The intuition of why SFW-asyn requires a small batch size for each iteration is that, while the error of asynchronous update already dominates the error for each iteration (by a factor of $\tau$), it will not change the order-wise behavior by having the gradient variance to be in the same scale as the error introduced by asynchronous update.

Our analysis, detailed in the appendix, is based on perturbed iterate analysis. This analysis is generalizable: with minor modification, we can show that SVRF-asyn converges, as in the below theorem.

\begin{theorem}
\textbf{(The convergence rate of SVRF-asyn)}. Consider problem \ref{eqn:main-problem}, where $F(\bX)$ is convex and $L$-smooth over the nuclear-norm ball. Denote $h_k = F(\bW_k) - F^*$, where $\bW_k$ is the output of the SVRF-asyn at its $k$ outer iteration. If $m_k = \frac{96 (k+1)}{\tau}$, $\eta_k = \frac{2}{k+1}$, $N_t = 2^{t+3} - 2$, the maximum delay is $\tau$, then
\begin{equation}
\expect \left[ h_k\right] \leq \frac{(3\tau + 12)\cdot 4 L D^2}{2^{k+1}}    
\end{equation}
\label{thm:SVRF-asyn}
\end{theorem}
We defer the proof to the appendix. This convergence rate is the same as in original SVRF \citep{hazan2016variance}.


\subsection{Constant Batch Size}

As discuss in the previous sub-section, the necessity of increasing batch size is not brought up by our asynchronous modification, but in the original Stochastic Frank-Wolfe method \citep{hazan2016variance}. 
However as the size of the mini-batch increases, it may get close to the size of the entire dataset, which often violates the interests of practical implementation. 
Below we briefly discuss two side-steps: 
(1) When the problem has the \textit{shrinking gradient variance}, the requirement of increasing batch size can be lifted 
(2) when only an approximately good result is needed (which is often the case when FW method is used), a fixed batch size will lead to a convergence to the optimal neighbourhood, with the same iteration convergence rate, for both the original SFW and our proposed SFW-asyn.

We show first that as long as the variance of the stochastic gradient is diminishing as the algorithm proceeds, the requirement of increasing batch size can be lifted.

\begin{definition}
\textbf{(shrinking gradient variance.)} A function $F(\bX) = \frac{1}{n} f_i(\bX)$ has shrinking gradient variance if the following holds:
\begin{equation}
    \expect \left[ \| \nabla f_i(\bX_k) - \nabla F(\bX_k)\|_F^2 \right] \leq \frac{G^2}{(k+2)^2}
\end{equation}
 as $\expect \left[ \| \bX_k - \bX^*\|_F \right]\leq \frac{c}{k+2}$ for some constant $c$ and $G$, and for an optimal solution $\bX^*$.
\end{definition}

For example, it is easy to verify that the noiseless matrix completion problem and the noiseless matrix sensing problem have the property of shrinking gradient variance.

It is easy to check that, under the \textit{shrinking gradient variance} condition, one can change the batch size requirement from $m_i = \frac{G^2 (i+1)^2}{ L^2 D^2}$ to $m = \frac{G^2 (c)^2}{ L^2 D^2}$ for SFW and from $m_i = \frac{G^2 (i+1)^2}{\tau^2 L^2 D^2}$ to $m = \frac{G^2 (c)^2}{\tau^2 L^2 D^2}$ for SFW-asyn, and maintain the same convergence rate as in Theorem 3 in \citep{hazan2016variance} and Theorem \ref{thm:conv-SFW-asyn} in the previous subsection.


While the shrinking gradient variance property is not presented, having a constant batch size will have SFW and SFW-asyn converges to a local neighbourhood of the optimal value:
\begin{theorem}
\textbf{(SFW converges to a neighbour of the optimal with constant batch size)}. Consider problem \ref{eqn:main-problem}, where $F(\bX)$ is convex and $L$-smooth over the nuclear-norm ball. Denote $h_k = F(\bX_k) - F^*$, where $\bX_k$ is the output of the $k^{th}$ iteration of SFW. If for all $i < k$, $m_i = \frac{G^2 c^2}{L^2 D^2}$ for a constant $c$, $\eta_i = \frac{2}{i+1}$, then
\begin{equation}
\expect \left[ h_k\right] \leq \frac{4 L D^2}{k+2} + \frac{1}{c} LD^2    
\end{equation}
\label{thm:conv-SFW-constant-batch}
\end{theorem}

The first term, $\frac{4 L D^2}{k+2}$ is inline with the convergence rate of the original SFW \citep{hazan2016variance}. 
Having a fixed batch size will incur a residual error $\frac{1}{c} LD^2 $ controlled by the batch size $c$. Similar convergence result holds for SFW-asyn:

\begin{theorem}
\textbf{(SFW-asyn converges to a neighbour of the optimal with constant batch size)}. Consider problem \ref{eqn:main-problem}, where $F(\bX)$ is convex and $L$-smooth over the nuclear-norm ball. Denote $h_k = F(\bX_k) - F^*$, where $\bX_k$ is the output of the $k^{th}$ iteration of SFW-asyn. If for all $i < k$, $m_i = \frac{G^2 c^2}{\tau^2 L^2 D^2}$ for a constant $c$, $\eta_i = \frac{2}{i+1}$, $\tau < T/2$,
\begin{equation}
\expect \left[ h_k\right] \leq \frac{(4\tau + 1)\cdot 2 L D^2}{k+2} + \frac{\tau}{c} LD^2
\label{eqn:conv-SFW-asyn-constant-batch}
\end{equation}
\label{thm:conv-SFW-asyn-constant-batch}
\end{theorem}
Note that the batch size in Theorem \ref{thm:conv-SFW-asyn-constant-batch} is $\tau^2$ times smaller than the batch size in Theorem \ref{thm:conv-SFW-constant-batch}.
The proof of the above two theorems can be found in the appendix.

\begin{corollary}
\label{coro:total_complexity}
\textbf{(Complexity to reach $\epsilon$ accuracy with fixed batch size)} For Algorithm \ref{alg:SFW-asyn} to achieve $F(\bX) - F^* \leq \epsilon$, we need 
$\Ocal \left(
\frac{\tau}{\epsilon - \tau/c}
\right)$
iterations in the master node.
It means, among all the machines, in total, we need
$\Ocal \left(
\frac{c^2}{\tau \epsilon - \tau^2/c}
\right)$
stochastic gradient evaluations, 
and 
$\Ocal \left(
\frac{\tau}{\epsilon - \tau/c}
\right)$
times linear optimization ($1$-SVD).
\end{corollary}
We defer the proof to the Appendix.

\begin{table}[h]
\begin{center}
\begin{tabular}{|c|c|c|}
\hline
& \textbf{\# Sto. Grad.} & \textbf{\# Lin. Opt.} \\ \hline
\textbf{SFW-asyn} 
& $\Ocal \left(\frac{c^2}{\tau \epsilon - \tau^2/c}\right)$
& $\Ocal \left(\frac{\tau}{\epsilon - \tau/c}\right)$
\\ \hline
\textbf{SFW}     
& $\Ocal \left(\frac{c^2}{\epsilon - 1/c}\right)$
& $\Ocal \left(\frac{1}{\epsilon - 1/c}\right)$
\\ \hline
\end{tabular}
\end{center}
\caption{Complexity comparison between SFW-asyn and SFW \citep{hazan2016variance} with fixed batch size (defined by the constant $c$, see Theorem \ref{thm:conv-SFW-constant-batch} and \ref{thm:conv-SFW-asyn-constant-batch}). \textbf{\# Linear Opt.} is the abbreviation of \textbf{the number of linear optimization} and \textbf{\# Sto. Grad.} is the abbreviation of \textbf{the number of stochastic gradient evaluation}.}
\label{tab:operations}
\end{table}

How good are these complexity results?
To interpret Table \ref{tab:operations} in a simple way, one could consider the case using a very large batch size $c$, and therefore SFW-asyn roughly reduces the stochastic gradient evaluations to $\frac{1}{\tau}$ portion of SFW, and requires $\tau$ times of linear optimizations. This is a good trade-off between the two processes considering we tackle the problems with very large-scale data-set where the stochastic gradient evaluation will dominate the computation for each iteration. In this sense, by simply viewing the number of operations required in each model update, SFW-asyn is already on par with or better than vanilla SFW, not including the speed-ups due to distributed computations with multiple workers. 


%% file: 6experiment.tex
\section{Empirical Results}
\label{sec:simulation}


\begin{figure*}[!thb]
\begin{center}
\includegraphics[width=0.24\linewidth,height=0.18\linewidth]{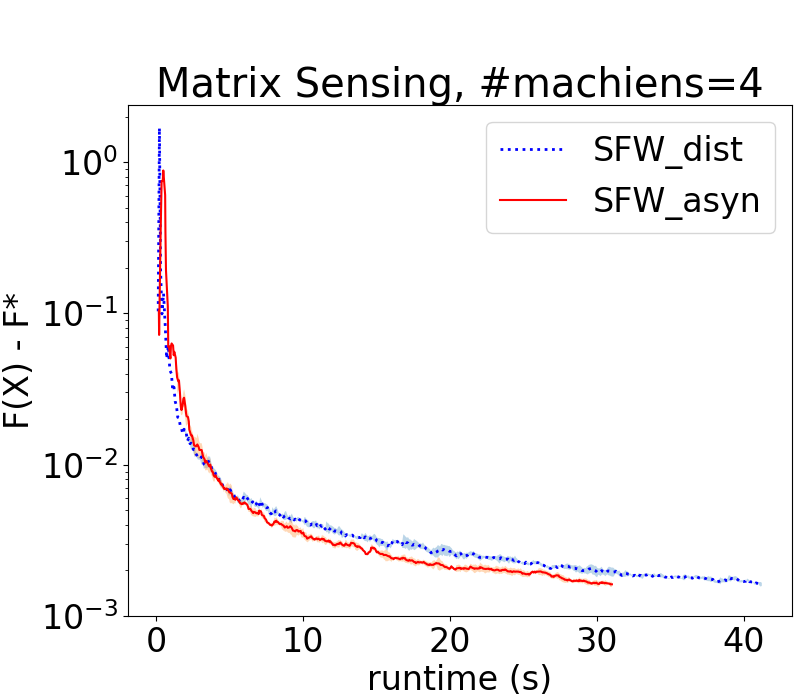}
\includegraphics[width=0.24\linewidth,height=0.18\linewidth]{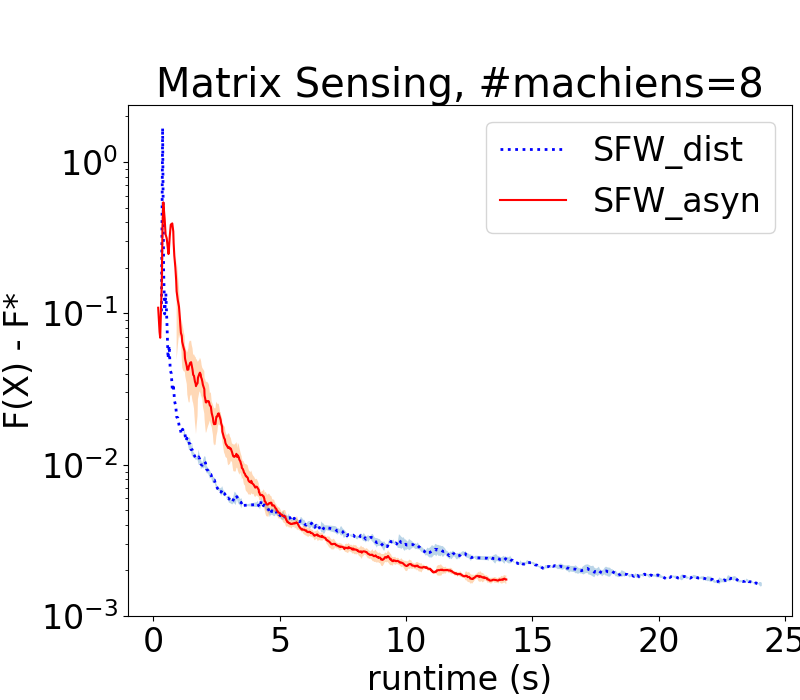}
\includegraphics[width=0.24\linewidth,height=0.18\linewidth]{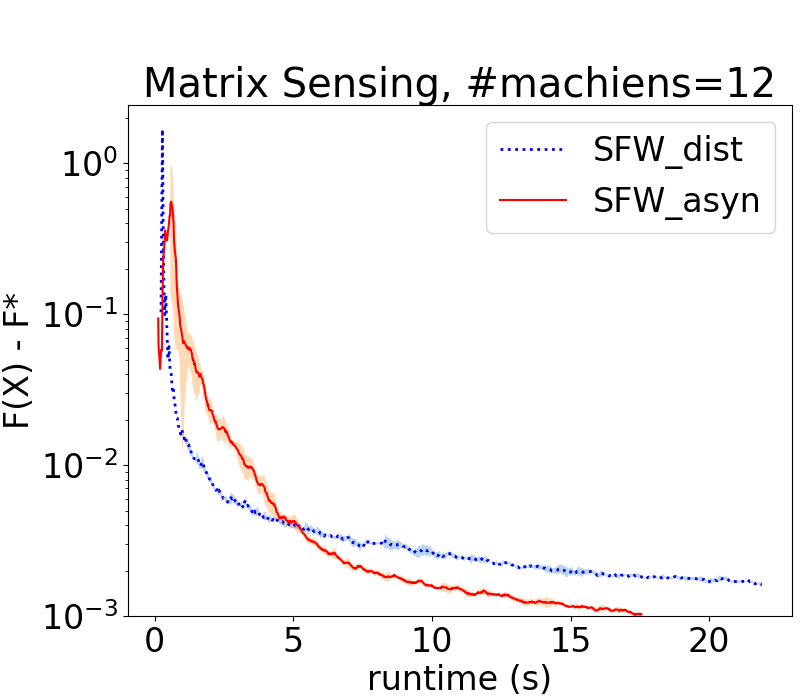}
\includegraphics[width=0.24\linewidth,height=0.18\linewidth]{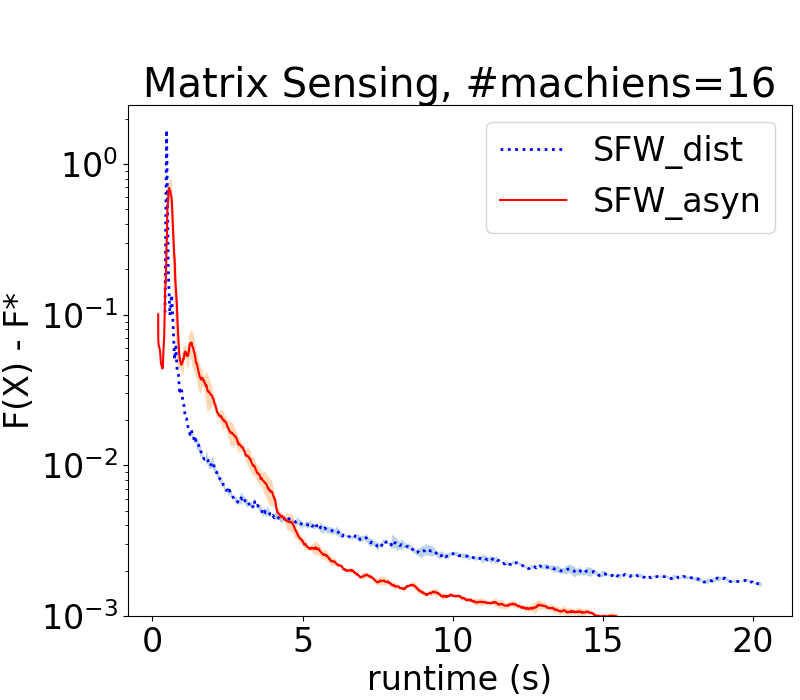}
\includegraphics[width=0.24\linewidth,height=0.18\linewidth]{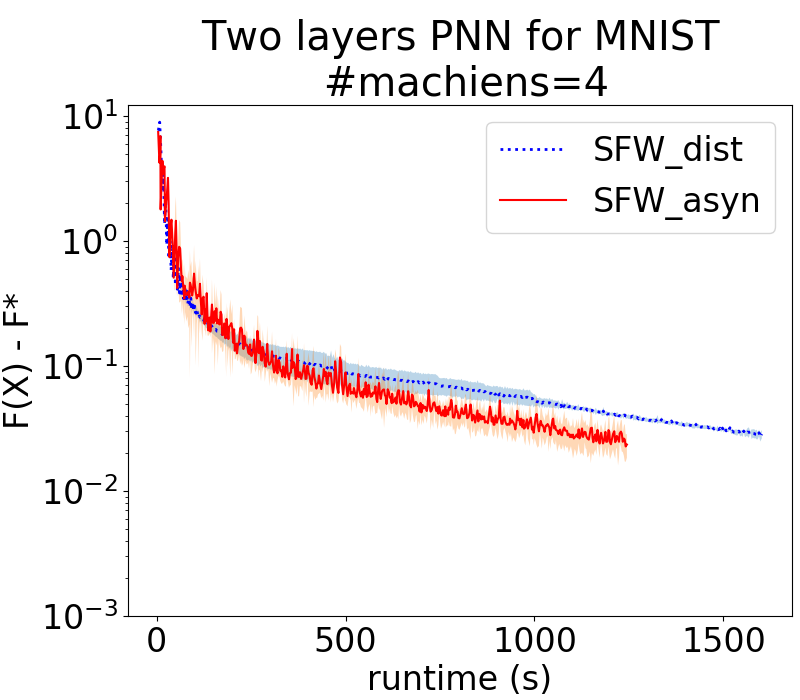}
\includegraphics[width=0.24\linewidth,height=0.18\linewidth]{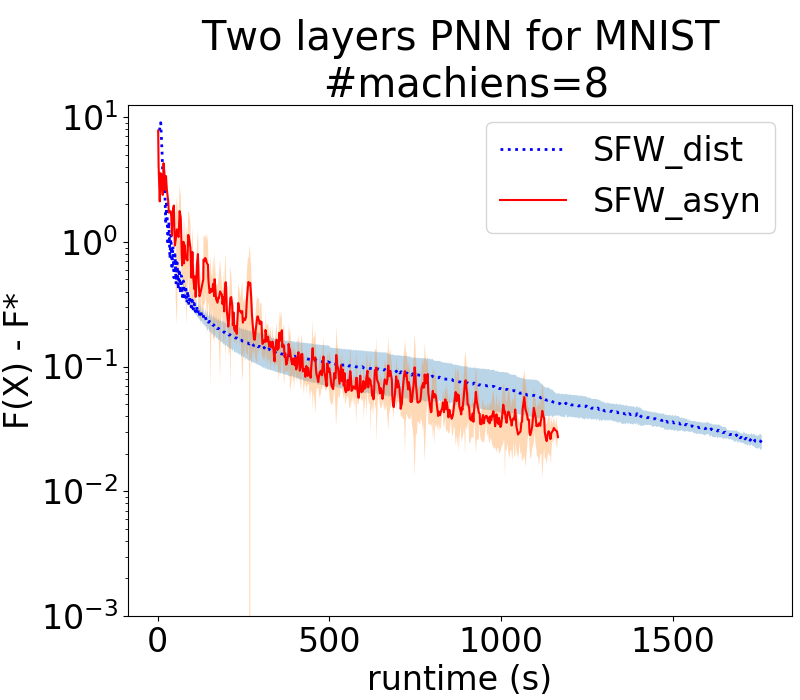}
\includegraphics[width=0.24\linewidth,height=0.18\linewidth]{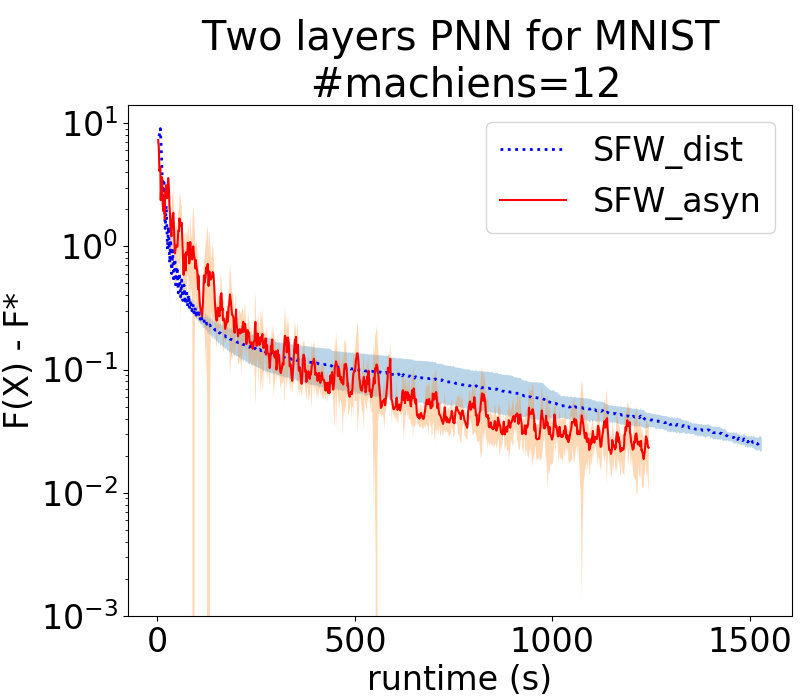}
\includegraphics[width=0.24\linewidth,height=0.18\linewidth]{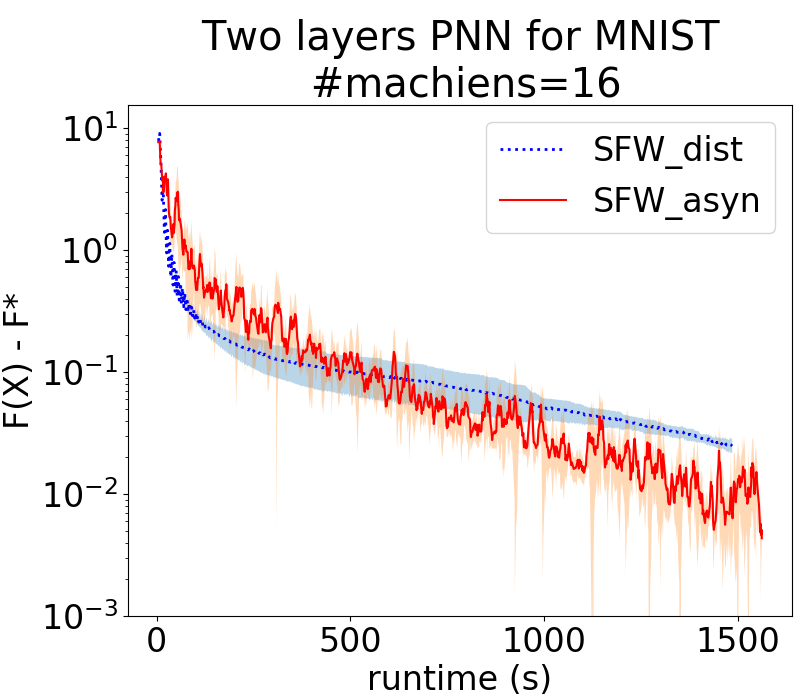}
\end{center}
\caption{Convergence of the relative loss vs runtime on AWS EC2 cluster. The first row shows the results for Matrix Sensing problem on synthesized data, and the second row shows that results for PNN on MNIST classification. The range of one standard deviation is shown as a shadow overlay. }
\label{fig:conv} 
\end{figure*}

In this section we empirically show the convergence performance of SFW-asyn, as supported by our theorems. We also show the speedup of SFW-asyn over SFW-dist on AWS EC2.

\subsection{Setup}

First let's review two machine learning applications:\\ 
\noindent\textbf{Matrix sensing}. Matrix sensing problem is to estimate a low rank matrix $\bX^*$, with observations of sensing matrices $\bA_i$ and sensing response $y_i = \langle \bA_i, \bX^* \rangle$, for $i = 1...N$. It is an important problem that appears in various fields such as quantun computing, image processing, system design, and so on (see \citep{park2016non} and the reference there in). Its connection to neural network is also an active research topic \citep{li2017algorithmic}.\\ \noindent\textbf{Polynomial Neural Network (PNN)}. PNN is the neural network with polynomial activation function. PNN has been shown to have universal representation power just as neural networks with sigmoid and ReLU activation function \citep{livni2014computational}.

We test the performance of SFW-asyn on minimizing the empirical risk of the matrix sensing problem with synthesized data: (1) generate the ground truth matrix $\bX^* = \bU \bV^T / \|\bU \bV^T\|_*$ where $\bU, \bV \in \mathbb{R}^{30\times 3}$ are generated by uniformly sampling from $0$ to $1$ for each entry; (2) generate $N=90000$ sensing matrices $\bA_i$ by sampling each entry from a standard normal distribution (3) compute response $y_i = <\bA_i, \bX^*> + \epsilon$ where $\epsilon$ is sampled from a Gaussian distribution with mean $0$ and standard deviation $0.1$. The empirical risk minimization for this task is
\[
\min_{\|\bX\|_* \leq 1} F(\bX) = \frac{1}{N} \sum_i^N (\langle \bA_i, \bX \rangle - y_i)^2.
\]

We also test SFW-asyn on minimizing the training loss of a Two layers PNN with quadratic activation function and smooth hinge loss to classify MNIST dataset of handwritten digits. In MNIST we have $60000$ training data $(\ba_i, y_i)$, where $\ba_i$ are vectorized $28*28$ pixels images, and $y_i$ are the labels. We set $y_i = -1$ if the label is $0, 1, 2, 3, 4$ or $y_i = 1$ otherwise. We divide the data $\ba_i$ by $255$ and therefore they are within zero and one. Hence the training objective of this neural network is
\[
\min_{\|\bX\|_* \leq \theta} F(\bX) = \frac{1}{N} \sum_i^N \textrm{s-hinge}(y_i, \ba_i^T \bX \ba_i),
\]
where the smooth hinge loss function $\textrm{s-hinge}(y, t)$ equals $0.5 - ty$ if $ty \leq 0$,  equals $ (0.5 \cdot (1-ty))^2$ if $0 \leq ty \leq 1$, and equals $0$ otherwise. While one can tune $\theta$ for good classification performance, we only show the results for $\theta=1$ as we are only interested in minimizing the objective value.

The hyper-parameters are mostly chosen as indicated by our analysis in Section \ref{sec:analysis}. We set the maximum batch-size to be $10000$ for the matrix sensing, and $3000$ for the PNN, such that the gradient computation time dominates the 1-SVD computation.

The distributed computation environment is set-up on AWS EC2. We leave the details in the Appendix.

\subsection{Result analysis}

We show the convergence results against the wall clock time on AWS EC2 in Figure \ref{fig:conv}, and show the speedup against single worker in Figure \ref{fig:speedup}.

Both SFW-dist and SFW-asyn obtain a better speedup result for the Matrix Sensing problem than PNN. SFW-dist has up to five times speedup for the Matrix Sensing problem, but only has marginal speedup on PNN problem, because the variable matrix size is much larger in the PNN problem than the Matrix Sensing problem. The variable matrix size in the Matrix Sensing problem is $30 \times 30 = 900$, while in PNN it is $784 \times 784 \approx 640 k$. And hence the speedup is quickly counteracted by the communication overhead. Although the communication cost of SFW-asyn is not as sensitive as SFW-dist to the increase of the matrix size, a larger variable matrix does increase the computation cost of the linear optimization. As in our Corollary \ref{coro:total_complexity}, SFW-asyn might comparatively perform more linear optimization, and hence SFW-asyn compromises in terms of the overall speedup ratio.

Nevertheless, the performance of SFW-asyn consistently outperforms SFW-dist.

\begin{figure}[t]
\begin{center}
\includegraphics[width=0.30\linewidth]{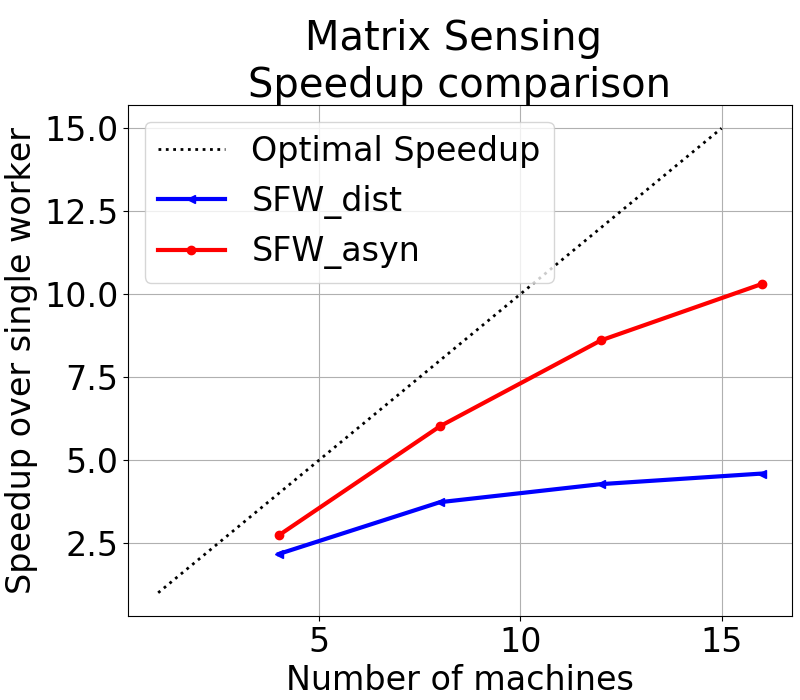}
\includegraphics[width=0.30\linewidth]{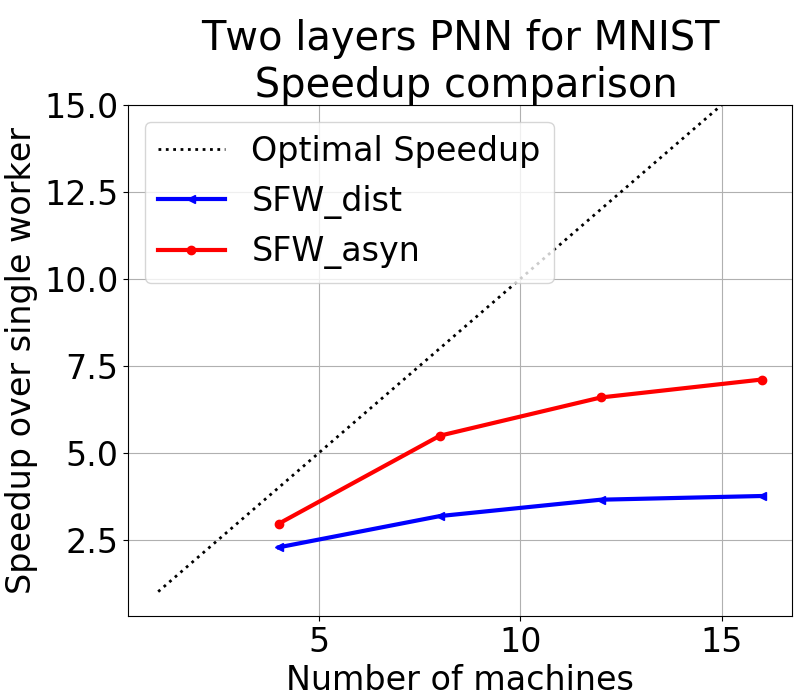}
\end{center}
\caption{Comparing the time needed to achieve the same relative error ($0.02$ for PNN and $0.001$ for Matrix Sensing) against single worker.}
\label{fig:speedup} 
\end{figure}



%% file: 7conclusion.tex
\section{Conclusion}

In this work we propose the Asynchronous Stochastic Frank-Wolfe (SFW-asyn) method, which simultaneously addresses the synchronization slow-down problem, and reduces the communication cost from $\mathcal{O} (D_1 D_2)$ to  $\mathcal{O} (D_1 + D_2)$. We establish the convergence guarantee for SFW-asyn that matches vanilla SFW, and experimentally validate the significant speed-ups of SFW-asyn over the baseline approaches on Amazon EC2 instances.

\subsection*{Acknowledgement}

This research has been supported by NSF Grants 1618689, DMS 1723052, CCF 1763702,
AF 1901292 and research gifts by Google, Western Digital and NVIDIA.

%% file: supp_proof.tex
\onecolumn

\section{Perturb Iterate Analysis of Frank Wolfe Type Methods}

While our algorithms and the analysis can be applied to general $\theta$, we focus on the case when $\theta = 1$ for notation simplicity.
Note there could be multiple optimal $\bX$, and hence we use $\bX^*$ to denote one of the optimal $\bX$ out of the whole optimal set, unless otherwise specified.
Let's start with a general definition of Frank-Wolfe Method framework.
The update iteration of Frank-Wolfe Method framework is
\begin{equation}
\begin{split}
    \bU_k &= \argmin_{\bU} \quad \langle  \gtil_{k-1}, \bU \rangle, \\
    \bX_k &= (1 - \eta_k) \bX_{k-1} + \eta_k \bU_k,
\end{split}
\label{eqn:general_fw}
\end{equation}
where $\gtil_{k-1}$ is the estimate of the gradient at point $\bX_{k-1}$, and $\eta_k$ is the step size at iterate $k$. For example, for SFW, $\gtil_{k-1}$ is $\frac{1}{|S|} \sum_{i \in S}\nabla f_i(\bX_{k-1})$, for SVRF $\gtil_{k-1}$ is the variance reduced stochastic gradient \citep{hazan2016variance}, and so on.

Note that the function $F$ is convex and $L$-smooth.
\begin{align*}
    F(\bX_k) &\leq 
    F(\bX_{k-1}) + \langle \nabla F (\bX_{k-1}), \bX_k - \bX_{k-1} \rangle + \frac{L}{2} \| \bX_k - \bX_{k-1}\|^2  \shortrnote{(by smoothness) }\\
    &\leq F(\bX_{k-1}) + \eta_k \langle \nabla F(\bX_{k-1}), \bU_k - \bX_{k-1} \rangle + \frac{L\eta_k^2}{2} \| \bU_k - \bX_{k-1}\|^2 \shortrnote{(by Eq. \eqref{eqn:general_fw})} \\
    &\leq F(\bX_{k-1}) + \eta_k \langle \nabla F(\bX_{k-1}), \bU_k - \bX_{k-1} \rangle + \frac{L\eta_k^2 D^2}{2} \shortrnote{(by Definition of D)} \\
\end{align*}
Let's look at the second term, which is the focus of our analysis.
\begin{align*}
    &\langle \nabla F(\bX_{k-1}), \bU_k - \bX_{k-1} \rangle\\
    =&  
    \langle \nabla F(\bX_{k-1}) - \gtil_{k-1}, \bU_k - \bX_{k-1} \rangle
    + \langle \gtil_{k-1} , \bU_k - \bX_{k-1} \rangle \\
    \leq& 
    \langle \nabla F(\bX_{k-1}) - \gtil_{k-1}, \bU_k - \bX_{k-1} \rangle 
    + \langle \gtil_{k-1} , \bX^* - \bX_{k-1} \rangle  \shortrnote{(by Eq. \eqref{eqn:general_fw})} \\
    =& 
    \langle \nabla F(\bX_{k-1}) - \gtil_{k-1}, \bU_k - \bX^* \rangle 
    + \langle \nabla F(\bX_{k-1}) , \bX^* - \bX_{k-1} \rangle  \\
    \leq & 
    F(\bX^*) - F(\bX_{k-1})  + \langle \nabla F(\bX_{k-1}) - \gtil_{k-1}, \bU_k - \bX^* \rangle . \shortrnote{(by the convexity of $F$)}
\end{align*}
Let $h_k = F(\bX_k) - F(\bX^*)$ and by putting everything together we get,
\begin{equation}
    h_k \leq (1-\eta_k) h_{k-1} + \eta_k \left[ \langle \nabla F(\bX_{k-1}) - \gtil_{k-1}, \bU_k - \bX^* \rangle  + \frac{L\eta_k D^2}{2} \right].
\label{eqn:pia}
\end{equation}

For the last term in the bracket, we call it residual. 
Note that $\frac{L\eta_k D^2}{2}$ can be controlled by diminishing step size $\eta_k$. 
The only remaining term to bound is $\langle \nabla F(\bX_{k-1}) - \gtil_{k-1}, \bU_k - \bX^* \rangle$, \textbf{which captures the inexactness of the gradient estimation}. 
If the gradient is exact, that is, $\langle \nabla F(\bX_{k-1}) - \gtil_{k-1}, \bU_k - \bX^* \rangle = 0$, the proof is then simple. 
One can plug in $\eta_k = 2 / (k+1)$ and show by induction that $h(k) \leq 4LD^2 / (k+2)$.
Indeed for our analysis we would like to have this term decrease with the rate of $\mathcal{O}(1/k)$, in order to retain the $\mathcal{O}(1/k)$ convergence rate.

\subsection{Convergence of SFW-asyn}

For Asyn-SFW, consider the worst case when a worker sends an update $\bu \bv^T$ based on $\bX_{\tau}$. That is,
\begin{equation}
\gtil_{k-1} =  \frac{1}{m_{k-\tau}}
    \sum_{i \in S_a}  \nabla f_i (\bX_{k-\tau}),
\label{eqn:gradient_estimate_sfw_asyn}
\end{equation}
for some sample set $S_a$.
We first setup an upper bound for $\langle \nabla F(\bX_{k-1}) - \gtil_{k-1}, \bU_k - \bX^* \rangle$. Then we combine it with Eq. \eqref{eqn:pia} to complete the proof for Theorem \ref{thm:conv-SFW-asyn}.

\begin{lemma}
Under the same assumptions of Theorem \ref{thm:conv-SFW-asyn}, 
let $\gtil_{k-1}$ be the gradient estimate defined in Eq. \eqref{eqn:gradient_estimate_sfw_asyn}, and
$\bU_k$ is defined in Eq. \eqref{eqn:general_fw}, 
then we can bound the inexactness of the gradient estimation for SFW-asyn:
\[
\expect \left [\langle \nabla F(\bX_{k-1}) - \gtil_{k-1}, \bU_k - \bX^* \rangle \right] \leq
\frac{GD}{\sqrt{m_{k-\tau}}} + L \tau \eta_{k-\tau} D^2
\]
\label{lemma:perturb-asyn-SFW}
\end{lemma}
\begin{proof}
\begin{align*}
& \langle \nabla F(\bX_{k-1}) - \gtil_{k-1}, \bU_k - \bX^* \rangle\\
\leq & \|  \nabla F(\bX_{k-1}) - \gtil_{k-1}\|_F \|\bU_k - \bX^*\|_F \\
\leq & D \| \nabla F (\bX_{k-1}) - \gtil_{k-1}\|_F  \\
= & \left\lVert \frac{1}{m_{k-\tau}} \sum_{i \in S_a}  \nabla f_i (\bX_{k-\tau}) 
- \nabla F (\bX_{k-\tau}) + \nabla F(\bX_{k-\tau}) - \nabla F (\bX_{k-1}) \right\rVert_F \cdot D \\
\leq & \left\lVert \frac{1}{m_{k-\tau}} \sum_{i \in S_a}  \nabla f_i (\bX_{k-\tau}) 
- \nabla F (\bX_{k-\tau}) \right\rVert_F \cdot D
+ \left\lVert \nabla F(\bX_{k-\tau}) - \nabla F (\bX_{k-1}) \right\rVert_F  \cdot D\\
\end{align*}
Hence the gradient inexactness is decomposed into two parts: the inexactness resulted from stochastic batch gradient (the first term), and the staleness (the second term). We increase the batch size and hence obtain a decrease rate of $1/k$ on the first part.
\begin{align*}
    \expect \left\lVert \frac{1}{m_{k-\tau}} \sum_{i \in S_a}  \nabla f_i (\bX_{k-\tau}) 
- \nabla F (\bX_{k-\tau}) \right\rVert^2_F \leq \frac{G^2}{m_{k-\tau}},
\end{align*}
and therefore
\begin{align*}
    \expect \left\lVert \frac{1}{m_{k-\tau}} \sum_{i \in S_a}  \nabla f_i (\bX_{k-\tau}) 
- \nabla F (\bX_{k-\tau}) \right\rVert_F \leq \frac{G}{\sqrt{m_{k-\tau}}},
\end{align*}
by Jensen's inequality.
For the stochastic inexactness, we have to control it by increasing batch size. Note that for some applications when the variance of the stochastic gradient shrinks as the algorithm approach the optimal point, the requirement of increasing batch size can be waived.

For the staleness:
\begin{align*}
    &\| \nabla F (\bX_{k-1}) - \nabla F (\bX_{k-\tau}) \|_F \\
    \leq& L \| \bX_{k-1} - \bX_{k-\tau}\|_F \shortrnote{(by smoothness)}\\
    =& L \| \bX_{k-1} - \bX_{k-2} + \bX_{k-2} + ... - \bX_{k-\tau}\|_F \\
    \leq& L \sum_{a=1}^{\tau+1} \| \bX_{k-a} - \bX_{k-a-1}\|_F \shortrnote{(by triangular inequality)}\\
    \leq& L \sum_{a=1}^{\tau+1} \eta_{k-a} \| \bU_{k-a} - \bX_{k-a-1}\|_F \shortrnote{(by Eq. \eqref{eqn:general_fw})}\\
    \leq& L \tau \eta_{k-\tau} D. 
\end{align*}

Combining the above two terms we finish the proof.
\end{proof}

Now back to the proof of Theorem \ref{thm:conv-SFW-asyn}.
\begin{proof}[Proof of Thm \ref{thm:conv-SFW-asyn}]
We finish the proof by induction. The case of $k=1$ is obviously true:
\begin{align*}
    F(\bX_1) - F^* &\leq 
    \langle \nabla F (\bX^*), \bX_1 - \bX^* \rangle + \frac{L}{2} \| \bX_1 - \bX^*\|^2  \shortrnote{(by smoothness) }\\
    &\leq 0 + \frac{L D^2 }{2}.\\
\end{align*}

Suppose $\expect \left[ h_k \right] \leq \frac{(4\tau+3)2LD^2}{k+2}$ for $k \leq T-1$. Then
\begin{align*}
\expect \left[ h_T \right] 
\leq & \expect \left[ (1-\eta_T) h_{T-1} + \eta_T \left( \langle \nabla F(\bX_{T-1}) - \gtil_{T-1}, \bU_T - \bX^* \rangle  + \frac{L\eta_T D^2}{2} \right) \right] \\
\leq & \expect \left[ (1-\eta_T) h_{T-1}\right] 
    + \eta_T \left( \frac{GD}{\sqrt{m_{T-\tau}}} + L \tau \eta_{T-\tau} D^2  + \frac{L\eta_T D^2}{2} \right) \shortrnote{(by Lemma \ref{lemma:perturb-asyn-SFW})}  \\
= & \expect \left[ \frac{T-1}{T+1} h_{T-1} \right] + \frac{2}{T+1} \left[ \frac{\tau LD^2}{T+1-\tau} + \frac{2\tau LD^2}{T+1 -\tau}   + \frac{LD^2}{T+1} \right] \shortrnote{(plug in $m_T$ and $\eta_T$)}  \\
\leq & \frac{T-1}{T+1} \frac{(3\tau+1)4LD^2}{T+1} + \frac{2}{T+1} \left[ \frac{\tau LD^2}{T+1-\tau} + \frac{2\tau LD^2}{T+1 -\tau}   + \frac{LD^2}{T+1} \right] \shortrnote{( recursion condition)}  \\
= & \frac{4LD^2}{(T+1)^2}\left[ \left(3\tau+1\right)(T-1) + \frac{T+1 }{T+1 -\tau}\left(\tau +\frac{\tau}{2}  \right)  + \frac{1}{2} \right]  \\
\leq & \frac{4LD^2}{(T+1)^2}\left[ \left(3\tau+1\right)(T-1) + 3\tau + 1 \right]  \shortrnote{($\tau < T/2$)}\\
= & \frac{4LD^2}{(T+1)^2}\left[ \left(3\tau+1\right)T \right]\\
\leq & \frac{(3\tau+1)4LD^2}{(T+2)}
\shortrnote{(since $\frac{T}{(T+1)^2} < \frac{1}{T+2}$)}\\
\end{align*} 

\end{proof}

\subsection{Convergence of SVRF-asyn}

For SVRF-asyn, consider the worst case when a worker sends an update $\bu \bv^T$ based on $\bX_{\tau}$. That is,
\begin{equation}
    \gtil_{k-1} =  \frac{1}{m_{k-\tau}}
    \sum_{i \in S_a}  \left[ \nabla f_i (\bX_{k-\tau}) - \nabla f_i (\bW) \right] + \nabla F (\bW),
\label{eqn:gradient_estimate_svrf_asyn}
\end{equation}

for some sample set $S_a$.

Before we get started, let's get prepared with some standard lemma on variance reduced algorithms.

\begin{lemma} (\textbf{Lemma 1 restated in \citep{hazan2016variance}})
For $\bX, \bW $ such that $\|\bX\|_* =  \|\bW\|_* = 1 $
\begin{align*}
    &\expect \left[ \|\nabla f_i(\bX) - \nabla f_i(\bW) + \nabla F(\bW) - \nabla F(\bX)\|^2 \right] \\
    \leq& 6 L \left(2 \expect [F(\bX) - F(\bW^*)] + \expect[F(\bW) - F(\bW^*)] \right)
\end{align*}
\end{lemma}


\begin{lemma}
Under the same assumptions as in Theorem \ref{thm:SVRF-asyn}, 
let $\gtil_{k-1}$ be the gradient estimate defined in Eq. \eqref{eqn:gradient_estimate_svrf_asyn},
$\bU_k$ defined as in Eq. \eqref{eqn:general_fw}. If
\[
\expect  \| \nabla F(\bX_{k-1}) - \gtil_{k-1} \|_F < \frac{15 \tau L D }{k + 1 - \tau}
\]
for $k \leq T$. Then for $k \leq T$
\[
\expect [F(\bX_k) - F^*] \leq \frac{(15\tau + 1) \cdot 4LD^2}{k+1}
\]
\end{lemma}
\begin{proof}
The proof is similar to the proof for Theorem \ref{thm:conv-SFW-asyn}. Let $h_k = F(\bX_k) - F^*$. We prove by induction.

\begin{align*}
\expect \left[ h_T \right]
\leq & \expect \left[ (1-\eta_T) h_{T-1} + \eta_T \left( \langle \nabla F(\bX_{T-1}) - \gtil_{T-1}, \bU_T - \bX^* \rangle  + \frac{L\eta_T D^2}{2} \right) \right] \\
= & \expect \left[ \frac{T-1}{T+1} h_{T-1} \right] + \frac{2}{T+1} \left[ \frac{L D^2 (15 \tau )}{T + 1 - \tau}   + \frac{LD^2}{T+1} \right] \shortrnote{(plug in $m_T$ and $\eta_T$)}  \\
\leq & \frac{T-1}{T+1} \frac{(15\tau+1)4LD^2}{T+1} + \frac{2}{T+1} \left[ \frac{L D^2 (15 \tau )}{T + 1 - \tau}   + \frac{LD^2}{T+1} \right] \shortrnote{(plug in $h_{T-1}$) }  \\
= & \frac{4LD^2}{(T+1)^2}\left[ \left(15\tau+1\right)(T-1) + \frac{1}{2} + \frac{T+1 }{T+1 -\tau}\frac{15}{2}\tau  \right]  \\
\leq & \frac{4LD^2}{(T+1)^2}\left[ \left(15\tau+1\right)(T-1) + \frac{1}{2} + 15\tau \right]  \shortrnote{($\tau < T/2$)}\\
\leq & \frac{4LD^2}{(T+1)^2}\left[ \left(15\tau+1\right)T \right]\\
\leq & \frac{(15\tau+1)4LD^2}{(T+2)}
\shortrnote{(since $\frac{T}{(T+1)^2} < \frac{1}{T+2}$)}\\
\end{align*} 
\end{proof}

The remain task is to setup an upper bound for $\| \nabla F(\bX_{k-1}) - \gtil_{k-1} \|_F$.

\begin{lemma}
Under the same assumptions as in Theorem \ref{thm:SVRF-asyn}, 
let $\gtil_{k-1}$ be the gradient estimate defined in Eq.\eqref{eqn:gradient_estimate_svrf_asyn},
$\bU_k$ defined as in Eq.\eqref{eqn:general_fw}, then we can bound the inexactness of the gradient estimation for SVRF-asyn:
\[\expect  \| \nabla F(\bX_{k-1}) - \gtil_{k-1} \|_F < \frac{L D (15 \tau )}{k + 1 - \tau}\]
\end{lemma}
\begin{proof}
For notation simplicity, we denote $\nabla_{k-1}$ as the variance reduced gradient. That is,
\[
\nabla_{k-1} = \frac{1}{m_{k-1}} \sum_{i \in S_b} \left( \nabla f_i(\bX_{k-1}) - \nabla f_i(\bW) \right) + \nabla F(\bW)
\]
We show this lemma by induction. The $k=1$ case is obvious. Suppose for $k < T-1$,
\[\expect  \| \nabla F(\bX_{k-1}) - \gtil_{k-1} \|_F < \frac{L D (15 \tau)}{k + 1 - \tau}.\]
By above lemma we have
\[\expect [F(\bX_k) - F(\bW^*)] \leq \frac{(15\tau + 1) \cdot 4LD^2}{k+1}.
\]

Then
\begin{align*}
    \| \nabla F(\bX_{T-1}) - \gtil_{T-1} \|_F  \leq  \|\nabla F(\bX_{T-1}) - \nabla_{T-1}\|_F + \|\nabla_{T-1} - \gtil_{T-1}\|_F. 
\end{align*}

Bounding $\|\nabla F(\bX_{T-1}) - \nabla_{T-1}\|_F$:
\begin{align*}
    & \expect \|\nabla F(\bX_{T-1}) - \nabla_{T-1}\|_F^2 \\
    \leq & \frac{6L}{m_{T-1}} \left(2 \expect [F(\bX_{T-1}) - F(\bW^*)] + \expect[F(\bW) - F(\bW^*)] \right)\\
    \leq & \frac{6L}{m_{T-1}} \left( \frac{8LD^2}{T+1} + \frac{LD^2}{2^t} \right) \cdot (15\tau + 1)\\
    \leq & \frac{6L}{m_{T-1}} \left( \frac{8LD^2}{T+1} + \frac{8LD^2}{T+1} \right)  \cdot (15\tau + 1)\\
    = & \frac{L^2 D^2}{(T+1)^2}  \cdot (15\tau^2 + \tau) \\
    \leq &  \frac{L^2 D^2}{(T+1)^2} \cdot 16 \tau^2
\end{align*}

Bounding $\|\nabla_{T-1} - \gtil_{T-1}\|_F$:
\begin{align*}
    &\expect \|\nabla_{T-1} - \gtil_{T-1}\|_F \\
    = & \expect \|\nabla_{T-1} - \nabla F(\bX_{T-1}) + \nabla F(\bX_{T-\tau}) - \gtil_{T-1} + \nabla F(\bX_{T-1}) - \nabla F(\bX_{T-\tau})\|_F \\
    \leq & \expect \|\nabla_{T-1} - \nabla F(\bX_{T-1}) \|_F + \expect \| \nabla F(\bX_{T-\tau}) - \gtil_{T-1}\|_F + \expect \| \nabla F(\bX_{T-1}) - \nabla F(\bX_{T-\tau})\|_F \\
    \leq & \frac{L D}{(T+1)}  \cdot \sqrt{16\tau^2} + \frac{L D}{(T+1 - \tau)} \cdot \sqrt{16\tau^2} + \frac{2 L D \tau}{(T+1 - \tau)} \\
    \leq & \frac{L D}{(T+1-\tau)}  \cdot (10\tau)  \\
\end{align*}
where the first inequality follows from the triangular inequality and the second inequality follows from Jensen's inequality, the bound of $\|\nabla F(\bX_{T-1}) - \nabla_{T-1}\|_F$, and the intermediate result in Lemma \ref{lemma:perturb-asyn-SFW}.

Therefore
\begin{align*}
    & \expect \| \nabla F(\bX_{T-1}) - \gtil_{T-1} \|_F  \\
    \leq & \expect \|\nabla F(\bX_{T-1}) - \nabla_{T-1}\|_F + \expect \|\nabla_{T-1} - \gtil_{T-1}\|_F \\
    \leq & \frac{L D}{(T+1)}  \cdot 4\tau + \frac{L D}{(T+1-\tau)}  \cdot 10 \tau \\
    < & \frac{L D(15\tau )}{T+1 - \tau} 
\end{align*}
where the first inequality follows from the triangular inequality and the second inequality follows from Jensen's inequality.
\end{proof}

By combining the above two Lemma, Theorem \ref{thm:SVRF-asyn} is easily established.

\subsection{Convergence with constant batch size}
\begin{proof}[Proof of Theorem \ref{thm:conv-SFW-constant-batch}]

We prove by induction. The case when $k=1$ is obvious. Suppose the theorem is true for $i \leq k-1$, then for $i=k$,
\begin{align*}
    \expect \left[ h_k \right] 
    \leq& \expect \left[ (1-\eta_k) h_{k-1} + \eta_k \left( \langle \nabla F(\bX_{k-1}) - \gtil_{k-1}, \bU_k - \bX^* \rangle  + \frac{L\eta_k D^2}{2} \right) \right]\\
    \leq& \expect \left[ (1-\eta_k) h_{k-1} + \eta_k \left( \frac{GD}{\sqrt{m}}  + \frac{L\eta_k D^2}{2} \right) \right]\\
    \leq & \frac{LD^2}{k+1} \left[ (k-1)\frac{\expect \left[h_{k-1}\right]}{LD^2} + \frac{2}{c} + \frac{2}{k+1} \right] \shortrnote{(plug in $m$ and $\eta_k$)} \\
    \leq& \frac{LD^2}{k+1} \left[ \frac{4(k-1)}{k+1} + \frac{k-1}{c} + \frac{2}{c} + \frac{2}{k+1} \right] \\
    =& \frac{LD^2}{k+1} \left[ \frac{4k-2}{k+1} + \frac{k+1}{c}  \right] \\
    \leq& \frac{4LD^2}{k+2} + \frac{1}{c} LD^2 \shortrnote{(since $\frac{k}{(k+1)^2} < \frac{1}{k+2}$)}\\
\end{align*}
\end{proof}

\begin{proof}[Proof of Theorem \ref{thm:conv-SFW-asyn-constant-batch}]

We prove by induction. The case when $k=1$ is obvious. Suppose the theorem is true for $k \leq T-1$, then for $k=T$,
\begin{align*}
h_T \leq &(1-\eta_T) h_{T-1} + \eta_T \left[ \langle \nabla F(\bX_{T-1}) - \gtil_{T-1}, \bU_T - \bX^* \rangle  + \frac{L\eta_T D^2}{2} \right] \\
\leq & (1-\eta_T) h_{T-1} + \eta_T \left[ \frac{GD}{\sqrt{m}} + L \tau \eta_{T-\tau} D^2  + \frac{L\eta_T D^2}{2} \right] \shortrnote{(similar to Lemma \ref{lemma:perturb-asyn-SFW})}  \\
= & \frac{T-1}{T+1} h_{T-1} + \frac{2}{T+1} \left[ \frac{\tau LD^2}{c} + \frac{2\tau LD^2}{T+1 -\tau}   + \frac{LD^2}{T+1} \right] \shortrnote{(plug in $m = \frac{c^2G^2}{L^2D^2}$ and $\eta_T$)}  \\
\leq & \frac{T-1}{T+1} \frac{(4\tau+1)2LD^2}{T+1} + \frac{T-1}{T+1} \frac{\tau LD^2}{c} + \frac{2}{T+1} \left[ \frac{\tau LD^2}{c} + \frac{2\tau LD^2}{T+1 -\tau}   + \frac{LD^2}{T+1} \right]   \\
= & \frac{4LD^2}{(T+1)^2}\left[ \left(2\tau+\frac{1}{2}\right)(T-1) + \frac{1}{2} + \frac{T+1 }{T+1 -\tau}\tau \right] + \frac{\tau LD^2}{c} \\
\leq & \frac{4LD^2}{(T+1)^2}\left[ \left(2\tau+\frac{1}{2}\right)(T-1) + \frac{1}{2} + 2\tau \right]  + \frac{\tau LD^2}{c}  \shortrnote{($\tau < T/2$)}\\
= & \frac{4LD^2}{(T+1)^2}\left[ \left(2\tau+\frac{1}{2}\right)T \right]  + \frac{\tau LD^2}{c}  \\
\leq & \frac{(4\tau+1)2LD^2}{(T+2)}  + \frac{\tau LD^2}{c}
\shortrnote{(since $\frac{T}{(T+1)^2} < \frac{1}{T+2}$)}\\
\end{align*} 
\end{proof}

\begin{proof}[Proof of Corollory \ref{coro:total_complexity}]

By Eq. \eqref{eqn:conv-SFW-asyn-constant-batch}, 
let 
$$\frac{(4\tau + 1)\cdot 2 L D^2}{k+2} + \frac{\tau}{c} LD^2 \leq \epsilon$$,
solve for $k$, and one can obtain the iteration bounds.
For each iteration, we need 
$$\Ocal(c^2/\tau^2)$$
stochastic gradient iteration. 
Multiply the iteration bound by $$\Ocal(c^2/\tau^2)$$
and we can see how much stochastic gradient evaluations that we need.
The number of linear optimization equals the number of total iterations.
\end{proof}

\newpage

%% file: supp_SVRF_asyn.tex
\section{Asynchronous Stochastic Variance Reduced Frank Wolfe}

\begin{algorithm}[t]
\caption{Naive Asynchronous Stochastic Variance Reduced Frank-Wolfe Method (SVRF-asyn) (Only for analysis, not implementation) }
\begin{algorithmic}[1]
\STATE // The Master Node 
\STATE \textbf{Input:} Max delay tolerance $\tau$; Max iteration count $T$; max inner-iteration counts $N_k$; Step size $\eta_t$ and batch size $m_t$
\STATE \textbf{Initialization:} Randomly initialize with $\|\bX_0\|_*=1$ and broadcast $\bX_0$.
\FOR{iteration $k=0,1,\cdots, T$}
    \STATE
    $\bX_{0} = \bW_k$, $t_m = 1$
    \WHILE{$t_m \leq N_k$}
        \STATE
        Wait until received $\{ \bU_w, t_w\}$ from a worker $w$.
        \STATE
        if $t_m - t_w > \tau$, abandon $\bU_w$ and continue.
        \STATE
        $t_m = t_m + 1$
        \STATE
        $\bX_{t_m} \leftarrow \eta_{t_m} \bU_{w} + (1-\eta_{t_m}) \bX_{t_m - 1}$
    \ENDWHILE
    \STATE
    $\bW_{k+1} = \bX_{N_k}$
    \STATE
    Broadcast $\bW_{k+1}$ and the update-$W$-signal
\ENDFOR
\STATE // For each worker $w=1,2,\cdots, W$
\WHILE{No Stop Signal}
    \IF{Update-$W$-signal}
        \STATE 
        Update the local copy of $\bW$ and Compute $\nabla F (\bW) $
    \ENDIF
    \STATE 
    Receive from the Master $\bX_{t_m}$.
    \STATE
    // $t_m$ is the inner iteration count at the master node.
    \STATE
    $t_w = t_m$, $\bX_w = \bX_m$.
    \STATE
    // Update the local copy of $\bX$ and iteration count.
    \STATE
    Randomly sample an index set $S$ where $|S| = m_{t_w}$
    \STATE 
    $\nabla_w = \frac{1}{m_{t_w}}\sum_{i \in S} \left(\nabla f_i (\bX_{t_w}) - \nabla f_i (\bW)\right) + \nabla F (\bW)$ // the variance reduced minibatch gradient
    \STATE 
    $\bU_w \leftarrow \textrm{argmin}_{\|\bU\|_*\leq \theta} \langle \nabla_w , \bU \rangle$ 
    \STATE
    send $\{\bU_w, t_w\}$ to the Master node.
\ENDWHILE
\end{algorithmic}
\label{alg:SVRF-asyn-naive}
\end{algorithm}

\begin{algorithm}[t]
\caption{Asynchronous Stochastic Variance Reduced Frank-Wolfe Method (SVRF-asyn)}
\begin{algorithmic}[1]
\STATE // The Master Node
\STATE \textbf{Input:} Max delay tolerance $\tau$; Max iteration count $T$; max inner-iteration counts $N_k$; Step size $\eta_t$ and batch size $m_t$
\STATE \textbf{Initialization:} Randomly initialize $\bX_0$ = $\bu_0 \bv_0^T$ s.t. $\|\bX_0\|_*=1$ and broadcast $\{\bu_0, \bv_0\}$ to all the workers; The iteration count at the master node $t_m = 0$.
\FOR{iteration $k=0,1,\cdots, T$}
    \STATE
    $\bX_{0} = \bW_k$, $t_m = 1$ // Maintain a local copy for output
    \WHILE{$t_m \leq N_k$}
        \STATE
        Wait until received from a worker $\{ \bu_w, \bv_w, t_w\}$.
        \IF{$t_m - t_w > \tau$}
            \STATE
            Send $(\bu_{t_m}, \bv_{t_m}), ... (\bu_{t_w+1}, \bv_{t_w+1})$ to node $w$.
            \STATE
            \textbf{continue}.
        \ENDIF
        \STATE
        $t_m = t_m + 1$ and store $\{ \bu_w, \bv_w\}$ as $\bu_{t_m}$ and $\bv_{t_m}$ 
        \STATE
        Send $(\bu_{t_m}, \bv_{t_m}), ... (\bu_{t_w+1}, \bv_{t_w+1})$ to node $w$.
        \STATE
        $\bX_{t_m} \leftarrow \eta_{t_m} \bu_{t_m} \bv_{t_m}^T + (1-\eta_{t_m}) \bX_{t_m-1}$ 
        \STATE
        // Not run in real time
        \STATE
        // Maintain a local copy for output
    \ENDWHILE
    \STATE
    $\bW_{k+1} = \bX_{N_k}$ // Maintain a local copy for output
    \STATE
    Send update-$W$-signal to all workers.
    \STATE
    Send $(\bu_N, \bv_N), ...$ to all workers.
\ENDFOR
\STATE // For each worker $w=1,2,\cdots, W$
\WHILE{No Stop Signal}
    \IF{Update-$W$-signal}
        \STATE 
        Obtain $(\bu_N, \bv_N), ...$ from the master 
        \STATE 
        Update the local copy of $\bX$ to $\bX_N$ as in Eqn \ref{eqn:update_to_k}
        \STATE 
        $\bW \leftarrow \bX_N$ and Compute $\nabla F (\bW) $
    \ELSE
        \STATE 
        Obtain $(\bu_{t_m}, \bv_{t_m}), ... (\bu_{t_w+1}, \bv_{t_w+1})$ from the master node.
        \STATE 
        Update the local copy of $\bX_{t_w}$ to $\bX_{t_m}$  according to Eqn \ref{eqn:update_to_k}
    \ENDIF
    \STATE
    Randomly sample an index set $S$ where $|S| = m_{t_w}$
    \STATE 
    $\nabla_w = \frac{1}{m_{t_w}}\sum_{i \in S} \left(\nabla f_i (\bX_{t_w}) - \nabla f_i (\bW)\right) + \nabla F (\bW)$ 
    \STATE 
    $\bu_w \bv_w^T \leftarrow \textrm{argmin}_{\|\bU\|_*\leq \theta} \langle \nabla_w , \bU \rangle$ 
    \STATE
    send $\{\bu_w, \bv_w, t_w\}$ to the Master node.
\ENDWHILE
\end{algorithmic}
\label{alg:SVRF-asyn}
\end{algorithm}

In this section we describe how to run SVRF asynchronously, and in a communication efficient way.
Similar to Section \ref{sec:algo}, we begin with a naive asynchronous SVRF, as in Algorithm \ref{alg:SVRF-asyn-naive}.
The core idea is to run the inner iteration of SVRF asynchronously.

And then we describe how to make the naive asynchronous SVRF communication efficient. We reduce the per-iteration communication cost of inner iterations to $\mathcal{O} (n)$, as what we achieved in SFW-asyn.

%% file: supp_simulation_setup.tex
\section{Distributed Computational Cluster Setup}

In this section we describe how to perform the experiments in Section \ref{sec:simulation} on Amazon AWS.

We use MIT StarCluster software as the heavy-lifting tool for AWS cluster management. 
We launch $15$ AWS \texttt{M1.SMALL} instances as the worker nodes, and $1$ AWS \texttt{M1.LARGE}, and connect these machines with a virtual private network.

We use \texttt{MPI4PY} \citep{dalcin2008mpi} as the fundamental APIs to implement distributed algorithms. \texttt{MPI4PY} is built on-top-of the Message Passing Standard, and is capable of implementing asynchronus distributed algorithms

%% file: supp_more_simulation.tex
\newpage
\section{More simulation results}

In order to better understand how much speedup asynchrony offers, and the conditions that the speedup occur, we test SFW-asyn against SFW-dist for matrix sensing and PNN under a distributed computational modeled by queuing theory. 

Queuing model is frequently used to model the staleness of each workers in distributed computational setting \cite{mitliagkas2016asynchrony}. 
We consider each $D_1 D_2$ operation takes one unit of time in expectation. 
Therefore, each stochastic gradient evaluation of matrix sensing and PNN takes one unit of time in expectation; 
we solve the 1-SVD up to a practical precision \cite{allen2017bfw}, 
and we consider 1-SVD takes ten units of time in expectation. 
In our simulation we find that setting the expected time of 1-SVD as ten or twenty or five has marginal impact on the results. 
We further assume that the computation time follows a geometric distribution:

\begin{assumption}
Denote random variable $t$ as the computation time required for each worker to finish a computation task that takes $C$ units of time in expectation. Then for $x = C, 2C, ... n C$, $\mathbb{P}(t = x)=p(1-p)^{x/C-1}$ for a distribution parameter $p$.
\end{assumption}

The intuition behind the staleness parameter $p$ is that, 
when $p$ is set to $1$, there is no randomness - each worker finish the work in the exactly same amount of time.
When the staleness parameter is small, say, $0.1$, 
then the computational time of each worker differs a lot - some may finish their jobs faster, while some are slower.

We want to emphasize that 

(1) this assumption is \textbf{not required} for our convergence results (Theorem \ref{thm:conv-SFW-asyn} and \ref{thm:SVRF-asyn}) to hold.

(2) the cost of communication is not taken into consideration, and hence we are implicitly favoring sfw-dist.

We show the convergence VS simulated time in Fig \ref{fig:simulated_time_append}, and the speedup over single worker in in Fig \ref{fig:simulated_speedup_append}. 
As in Fig \ref{fig:simulated_speedup_append}, the speedup of SFW-asyn is almost linear, while SFW-dist compromises as the number of workers get larger.

The improvement over SFW-dist is less significant, as we increase the staleness parameter from $p=0.1$ to $p=0.8$.
Obviously, SFW-dist performs better on large staleness parameter.
When the staleness parameter is closer to one, the machines performs more uniformly (i.e., they finish the mini-batch computation in the similar amout of time), and therefore the slow-dowm of SFW-dist due to slowest worker is less significant.
However, the performance of SFW-asyn also compromises slightly, as we go from $p=0.1$ to $p=0.8$.
That is, SFW-asyn slightly prefer random delay, rather than consistent delay.
This means that our analysis based on worst case scenario is loose, and can be improved.

\begin{figure*}[thb]
\begin{center}
\includegraphics[width=0.24\linewidth]{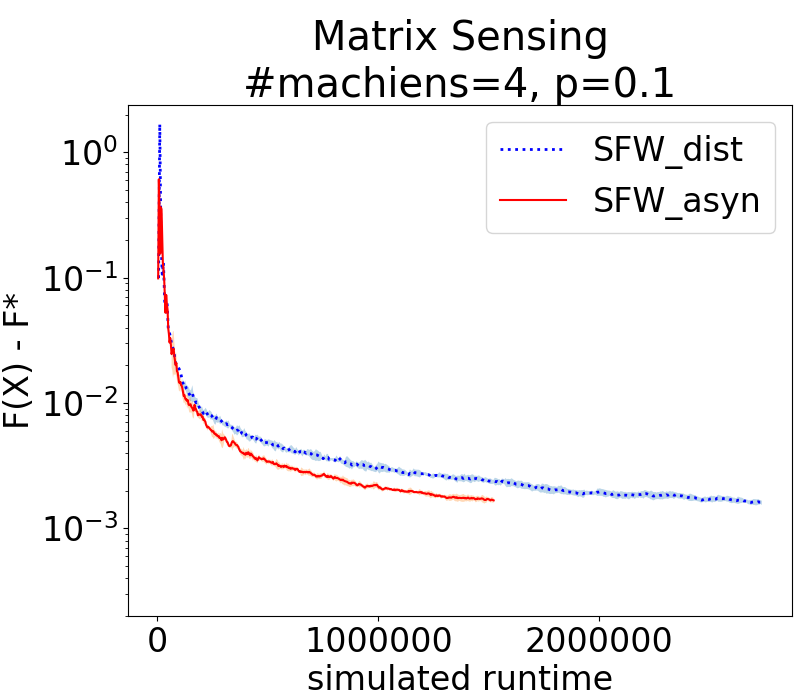}
\includegraphics[width=0.24\linewidth]{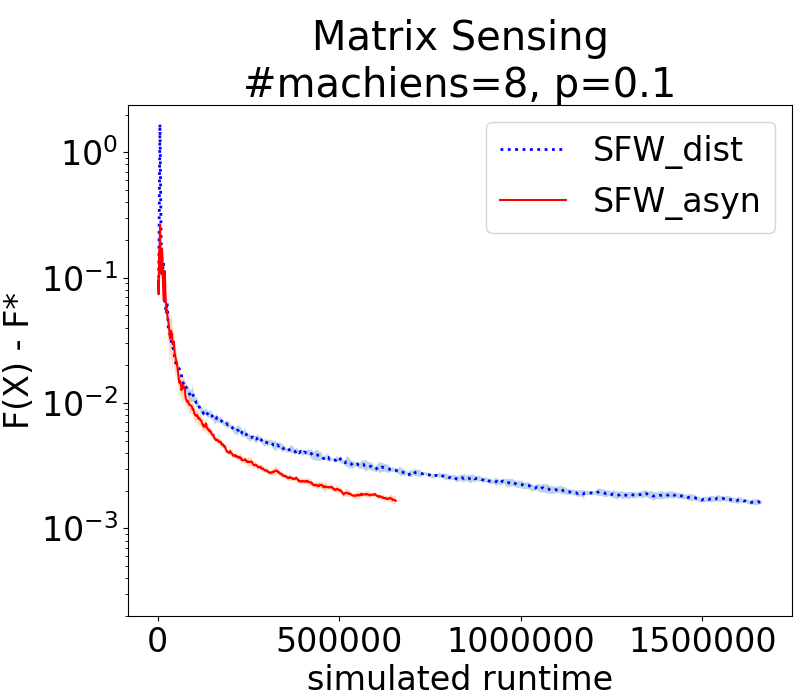}
\includegraphics[width=0.24\linewidth]{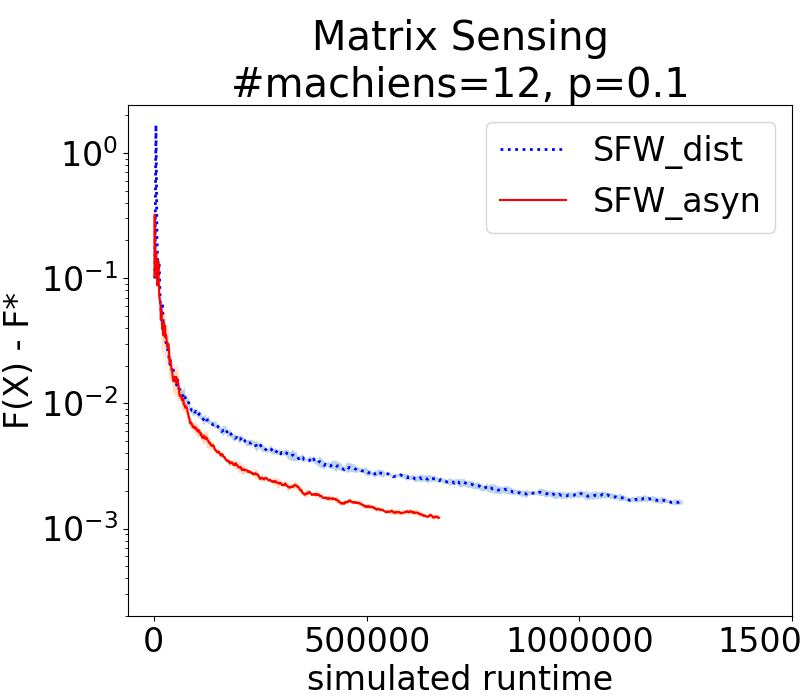}
\includegraphics[width=0.24\linewidth]{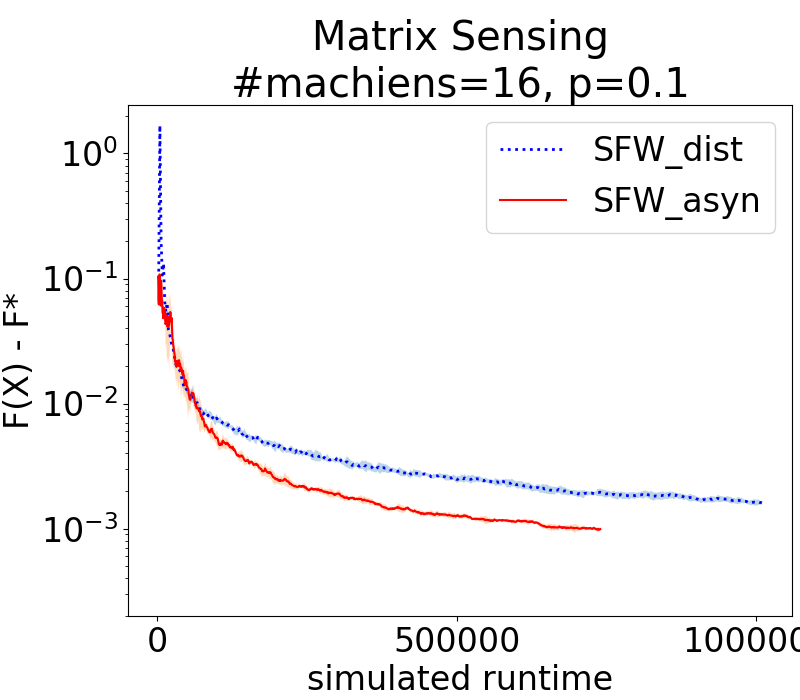}
\\ \vspace{5pt}
\includegraphics[width=0.24\linewidth]{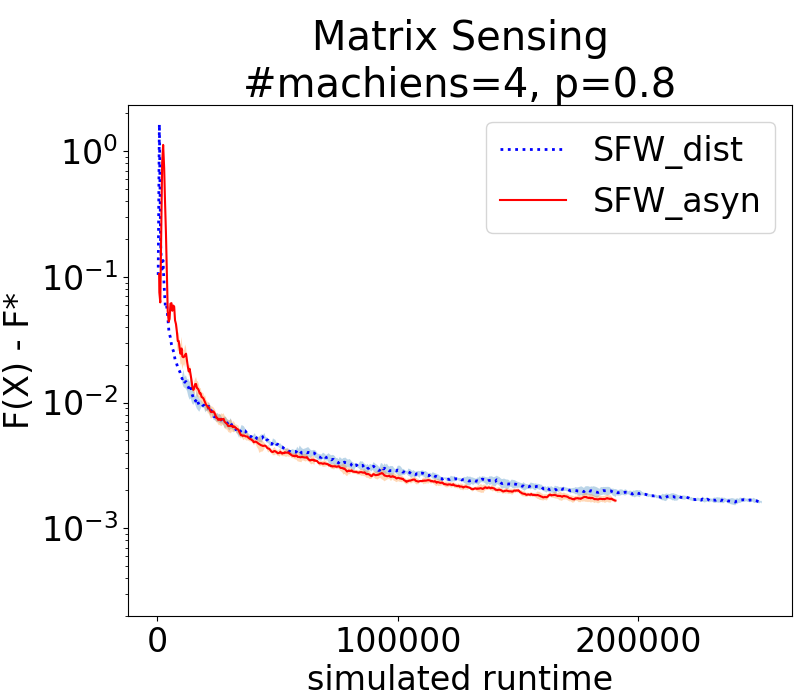}
\includegraphics[width=0.24\linewidth]{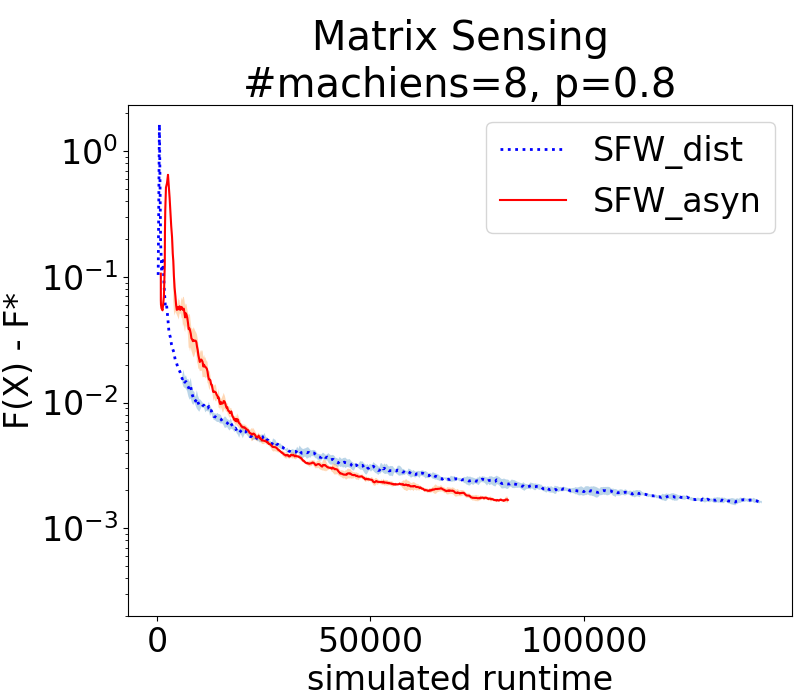}
\includegraphics[width=0.24\linewidth]{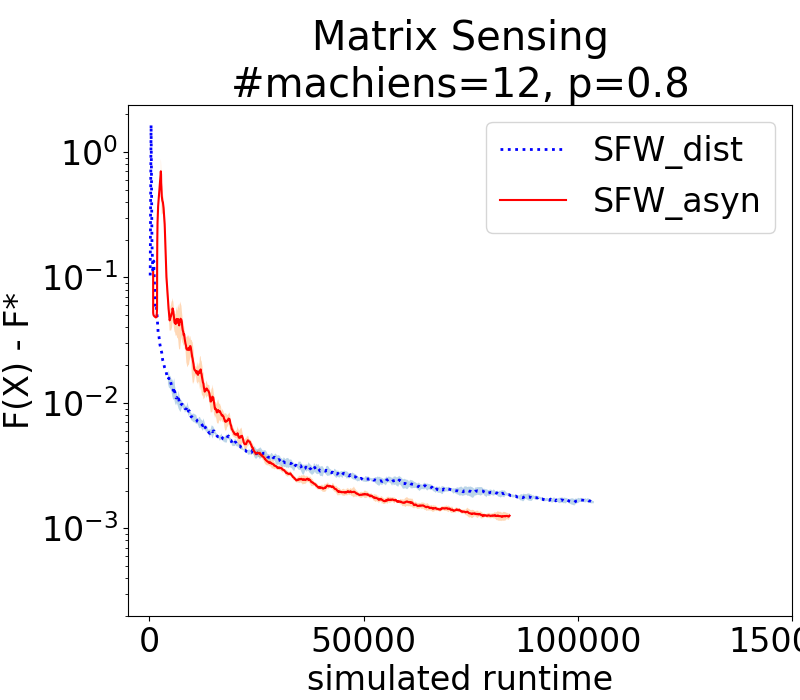}
\includegraphics[width=0.24\linewidth]{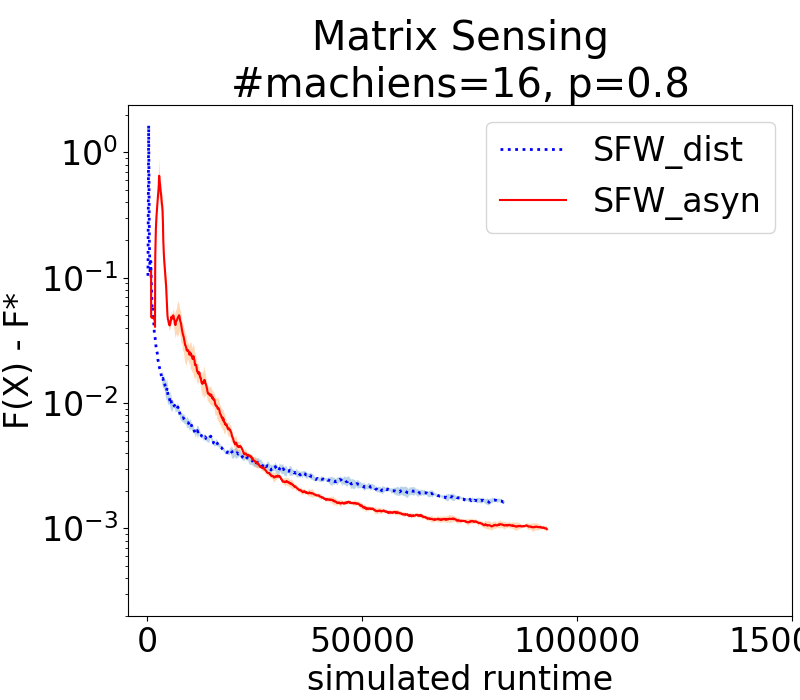}
\end{center}
\caption{Convergence of the relative loss of Matrix Sensing problem VS the time simulated by queuing model. Simulation are repeated $5$ times, and $1$ standard deviation is shown in the colored shadow.}
\label{fig:simulated_time_append} 
\end{figure*}

\begin{figure}[th]
\begin{center}
\includegraphics[width=0.37\linewidth]{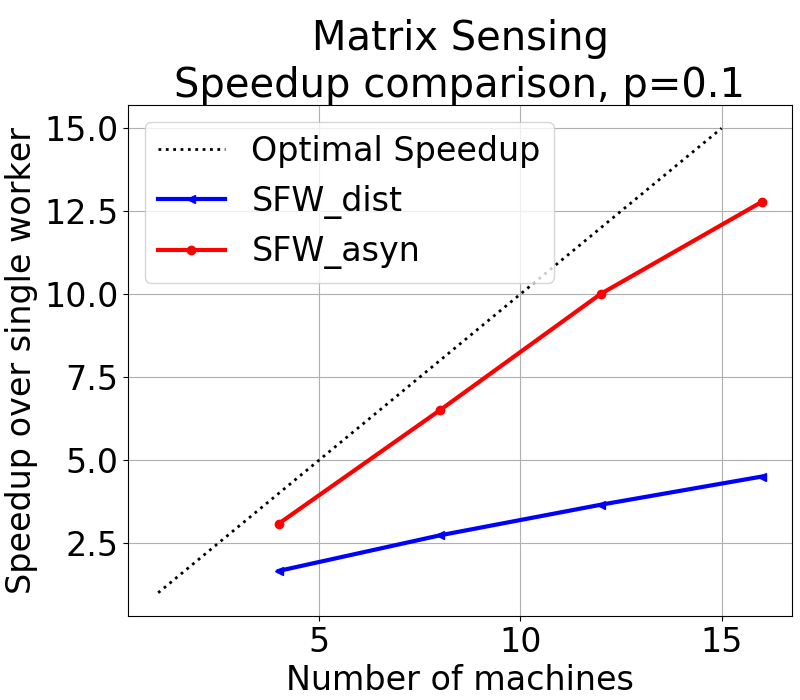}
\hspace{20pt}
\includegraphics[width=0.37\linewidth]{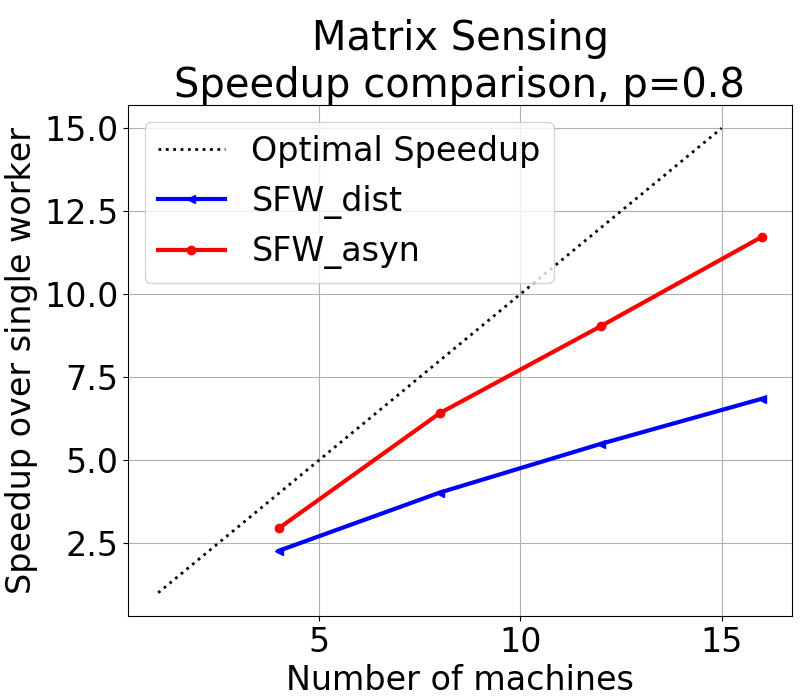}
\end{center}
\caption{Comparing the time needed to achieve the same relative error ($0.002$) against single worker.}
\label{fig:simulated_speedup_append} 
\end{figure}